%% file: main.tex
\RequirePackage[T1]{fontenc}
\documentclass[conference]{IEEEtran}
\usepackage[utf8]{inputenc}
\ifdefined\XeTeXversion
  \catcode`’=\active \def’{\textquoteright}
  \catcode`–=\active \def–{\textendash}
\fi
\usepackage{amsmath,amssymb,amsfonts,amsthm}
\usepackage{array,booktabs,siunitx}
\usepackage{graphicx,nicefrac,microtype}
\usepackage[numbers,sort&compress]{natbib}
\usepackage{url}
\usepackage{placeins}
\usepackage[hidelinks]{hyperref}
\hypersetup{pdfauthor={Vishwajith Ramesh},pdftitle={Exact Record Omission in Delta Attention: A Transport Criterion, Its Cost, and a Replay Certificate},pdfsubject={Public preprint}}
\newtheorem{theorem}{Theorem}

\newtheorem{corollary}{Corollary}
\theoremstyle{definition}

\title{Exact Record Omission in Delta Attention:\\A Transport Criterion, Its Cost, and a Replay Certificate}
\author{\IEEEauthorblockN{Vishwajith Ramesh}
\IEEEauthorblockA{Vy Labs, Inc.\\
\texttt{vish@vylabs.ai}}}
\input{kda_sweep_macros}

\begin{document}
\maketitle
\begin{abstract}
\input{abstract_kda}
\end{abstract}
\input{body_kda}
\input{satml_statements}
\bibliographystyle{IEEEtranN}
\bibliography{references}
\appendices
\input{appendix_kda}

\end{document}

%% file: kda_sweep_macros.tex
\providecommand{\FamExactThreshold}{1\times10^{-4}}
\providecommand{\FamFalconAdmitted}{12}
\providecommand{\FamFalconAttempted}{12}

\providecommand{\FamFalconForcingMedian}{1.04}

\providecommand{\FamFalconFullErrMax}{2.6\times10^{-6}}

\providecommand{\FamFalconLedgerErrMax}{2.6\times10^{-6}}
\providecommand{\FamFalconLiftFullMedianLast}{0.77}

\providecommand{\FamFalconLiftPresentMedianLast}{0.80}

\providecommand{\FamFalconStaticErrMax}{3.0}
\providecommand{\FamMambaAdmitted}{12}
\providecommand{\FamMambaAttempted}{12}

\providecommand{\FamMambaForcingMedian}{1.02}

\providecommand{\FamMambaFullErrMax}{1.6\times10^{-6}}

\providecommand{\FamMambaLedgerErrMax}{1.6\times10^{-6}}
\providecommand{\FamMambaLiftFullMedianLast}{-0.23}

\providecommand{\FamMambaLiftPresentMedianLast}{0.55}

\providecommand{\FamMambaStaticErrMax}{2.1}
\providecommand{\FamRwkvAdmitted}{12}
\providecommand{\FamRwkvAttempted}{12}

\providecommand{\FamRwkvForcingMedian}{1.00}

\providecommand{\FamRwkvFullErrMax}{1.9\times10^{-6}}

\providecommand{\FamRwkvLedgerErrMax}{0.191}
\providecommand{\FamRwkvLiftFullMedianLast}{-0.30}

\providecommand{\FamRwkvLiftPresentMedianLast}{0.79}

\providecommand{\FamRwkvRecurrenceErrMax}{0}

\providecommand{\FamRwkvStaticErrMax}{0.518}

\providecommand{\MlaConfigs}{8}
\providecommand{\MlaLayers}{7}
\providecommand{\MlaMaxAbsDiff}{0}
\providecommand{\MlaSources}{synthetic and TOFU}

\providecommand{\SampSynK}{100}

\providecommand{\SweepSynAdmitted}{24}
\providecommand{\SweepSynAttempted}{24}
\providecommand{\SweepSynCheckpointMiB}{41.4}
\providecommand{\SweepSynDecompErrMax}{9.0\times10^{-4}}
\providecommand{\SweepSynForcingShareMaxLast}{2.712}
\providecommand{\SweepSynForcingShareMedianFirst}{0.67}
\providecommand{\SweepSynForcingShareMedianLast}{0.98}
\providecommand{\SweepSynForcingShareMinLast}{0.000}
\providecommand{\SweepSynFullResidualMedianLast}{4.4}

\providecommand{\SweepSynImprintMaxLast}{24.2}
\providecommand{\SweepSynImprintMedianLast}{4.5}
\providecommand{\SweepSynImprintMedianZero}{30.8}
\providecommand{\SweepSynImprintMinLast}{0.3}

\providecommand{\SweepSynLedgerResidualMedianLast}{6.0}

\providecommand{\SweepSynLogKiBPerToken}{481}
\providecommand{\SweepSynLogMiBPerThousandTokens}{470}
\providecommand{\SweepSynLogTokensPerCheckpoint}{88}
\providecommand{\SweepSynMaxCut}{4{,}096}

\providecommand{\SweepSynNumCuts}{10}
\providecommand{\SweepSynPrefillMsPerToken}{1.32}
\providecommand{\SweepSynPrefixDepths}{0, 4, 16, 64}

\providecommand{\SweepSynReplayMaxDiff}{0}

\providecommand{\SweepSynStaticResidualMedianLast}{27.5}
\providecommand{\SweepSynTransportMsPerToken}{9.4}
\providecommand{\SweepSynTruncEighthMedian}{0.77}
\providecommand{\SweepSynTruncHalfMedian}{0.56}
\providecommand{\SweepSynTruncQuarterMedian}{0.69}
\providecommand{\SweepSynVictims}{6}
\providecommand{\SweepTofuAdmitted}{16}
\providecommand{\SweepTofuAttempted}{16}

\providecommand{\SweepTofuImprintMedianLast}{2.9}

\providecommand{\SweepTofuMaxCut}{4{,}096}

\providecommand{\SweepTofuPrefixDepths}{0, 16}

%% file: abstract_kda.tex
An assistant can stop repeating a deleted statement while its recurrent memory still carries that statement's influence. We examine this distinction by saving the state difference immediately after a record, carrying this saved difference, or \emph{receipt}, through later updates, and comparing the corrected state with the state built from the same conversation with the record omitted. Unrolling the recurrence gives an exact criterion: transport reaches this never-stored state if and only if the additional differences created by later updates cancel after transport. We evaluate eleven conditions using 40 prospectively selected synthetic Kimi Linear contexts, with 40 separate calibration contexts and paired uncertainty estimates computed across complete records. Attention masking removed all 40 exact greedy target answers, yet a fixed three-query candidate-scoring attack with known candidate values and log-probability access attained AUC \(0.740\) (95\% interval \(0.695\)--\(0.809\)); the measured false-positive rate was 7.5\%. Prompt-only forgetting still returned 33 of the 40 targets. Checkpoint replay matched the complete declared active state and audit logits in all 80 contexts, while preserving all 120 retained answers in the evaluation cohort. An independent Qwen cohort and matched original-bf16/8-bit Kimi controls confirmed native transport mismatch beyond the numerical error measured in matched controls. Replay also supported record replacement and successive deletions with a new record added between them. These results provide a practical audit of both the requested memory change and the assistant's remaining behavior, with replay work determined by the surviving suffix.

%% file: body_kda.tex
\section{Introduction}

In a case discussion, a physician tells an AI assistant that a patient has a provisional diagnosis of pneumonia, then shares the patient's symptoms and history. Later testing rules out pneumonia, so the physician asks the assistant to forget the earlier diagnosis while retaining follow up discussion. In this synthetic example, the assistant should continue helping with the patient case without treating pneumonia as a current diagnosis. Otherwise, the outdated diagnosis could bias its later summaries or suggestions to the physician. A reply such as ``I do not remember'' is insufficient evidence of forgetting: the assistant could avoid mentioning pneumonia while still retaining the earlier diagnosis and using it to interpret new information. We therefore compare its edited memory with the state the same model
would have built if the provisional diagnosis had been left out. We call this reference
the \emph{never-stored state}.\footnote{Code and source-free audit artifacts: \url{https://github.com/vishrmsh/delta-attention-record-omission/}.}

There is also a privacy question: could someone still learn what was removed?
We call someone trying to recover or detect the deleted information an
\emph{attacker}. They might ask for the old diagnosis directly. If the service
exposes probability scores, they might instead test a diagnosis they already
suspect.

By memory, we mean the running numerical state that a model builds as it reads a conversation at inference time; its pretrained parameters remain fixed. Our main test case is Kimi Linear, a released pretrained model that combines Kimi Delta Attention (KDA) with multi-head latent attention (MLA) \cite{kimilinear}. KDA was important because it helped to make long-context inference cheaper without relying on full attention at every layer. Instead of keeping a separate cache entry for every earlier token, KDA repeatedly folds the conversation into a compact recurrent state whose size does not grow with the conversation. Its channel-wise decay allows different parts of that state to fade at different rates. In Kimi Linear, recurrent KDA layers are interleaved with MLA layers; the original work reported better performance than a matched model using MLA throughout, together with up to 75\% lower KV-cache use and up to \(6\times\) higher decoding throughput at one-million-token context \cite{kimilinear}.

Those gains also create the deletion problem we study. Once a statement has been folded into a shared, evolving recurrent state, it no longer occupies a single token entry that can simply be removed. Kimi Linear is therefore a useful stress test rather than an arbitrary model choice. Its released implementation lets us inspect the recurrent KDA state and the token-associated MLA cache separately. We can mask a statement's MLA entries and ask whether its influence remains in KDA, and we can test what must be recomputed to match a run that never processed the statement. Experiments on additional model families examine how far these findings extend beyond Kimi Linear.

Consider the fictional physician example. Before asking the assistant to forget the provisional pneumonia diagnosis, the physician asks it to ``summarize the case.'' The request does not mention pneumonia, but the model processes it using a state already shaped by that diagnosis. The state produced after this later exchange may therefore differ from the state the model would have built had the diagnosis never appeared. By the time deletion is requested, the diagnosis may have changed both the initial memory and how later messages were processed. Subtracting only its original footprint can thus leave downstream effects behind.

We compare four ways of responding to the deletion request. The system operator can (1) tell the assistant to disregard the diagnosis, (2) hide the attention entries associated with its original statement, (3) subtract a saved numerical footprint of the statement, or (4) rewind to a checkpoint before the statement and replay the surviving conversation. These correspond to \emph{instruction-only forgetting}, \emph{attention masking}, \emph{receipt transport}, and \emph{checkpoint replay}.

In the physician example, imagine running the same model twice: once with the provisional diagnosis and once without it. We save the difference between their memories immediately after the diagnosis statement; we call this difference a \emph{receipt}. We then carry that receipt through later updates and subtract it from the final memory. The difficulty is that the diagnosis may have also changed the way the assistant processed the following symptoms and history. The correction reaches the never-stored state only if those additional differences cancel out at the end. Section~\ref{sec:criterion} gives the exact condition. Replay takes a more direct route: return to a saved state from before the diagnosis, skip that statement, and process the same recorded discussion again. Its cost grows with how much discussion must be processed again.

We evaluate these strategies under three levels of access. The first two
model what an attacker could learn; we use \emph{adversary} as the technical
term for this role. An \emph{extraction adversary} asks for the deleted value
directly. A \emph{candidate-scoring adversary} already has a suspected value
and requests its probability under a fixed set of queries. In the physician
example, this means testing ``pneumonia'' even when the assistant does not say
it. A \emph{privileged auditor}, by contrast, can process the same recorded
discussion with the diagnosis omitted. The auditor compares the resulting
memory and output scores with those of the edited assistant.
Section~\ref{sec:terms} defines these access levels and the reference used by
each comparison.

We evaluate the four strategies on the released 48B Kimi Linear hybrid using synthetic administrative records. Each context contains one target assertion to remove and other information to retain. The experiments measure two different outcomes: whether the assistant still reveals the target and whether its edited memory matches a run that never processed the target statement.

The difference is substantial. In 40 prospectively selected contexts, attention masking prevented exact target extraction in all 40. Nevertheless, an evaluator who already knew a candidate value could still distinguish the masked state from the never-stored reference, with an area under the receiver operating characteristic curve (AUCROC) of \(0.740\), where \(0.5\) denotes chance. A fixed-format instruction to forget was weaker: it still elicited the target value in 33 of 40 contexts. Checkpoint replay, by contrast, matched the never-stored reference on every audited internal-memory array and on the model's full next-token score vector across all 80 evaluation and calibration contexts. All 120 answers about retained records were preserved in each condition. A missing target answer therefore does not by itself show that the underlying memory has been restored.

We also test whether the receipt criterion holds across other recurrent model families and repeat the key Kimi comparisons at matched numerical precision. Beyond deleting a single record, we use replay to make twelve record corrections and to delete records in sequence, with a new record added between deletions. We measure how much checkpoint storage and repeated computation these operations require.

In this work, we make three contributions to the verification of inference-time memory deletion. First, we give an exact criterion for when transporting and subtracting a saved boundary receipt reproduces the never-stored recurrent state, and connect failure of that criterion to how the record changes the processing of later messages. Second, we show that three outcomes can diverge: a model may stop stating the target, a candidate-aware evaluator may still detect its influence, and the internal memory may remain different from the never-stored reference. Third, we evaluate checkpoint replay as an auditable procedure for deletion, replacement, and repeated memory maintenance, and quantify its storage and computation costs. Together, these results tell an operator what each editing procedure changes, what level of access can verify that change, and what the procedure costs.

\section{Terms and scope}\label{sec:terms}

We fix the model, tokenizer, arithmetic, update schedule, and initial state \(M_0\). Here, \(M\) includes all states needed to resume the computation, including relevant caches and position counters. For a history
\(H\), let \(B(H)\) be the state obtained by applying
\begin{equation}
M_t=U(M_{t-1},x_t),\qquad U_x(M):=U(M,x).
\label{eq:general-update}
\end{equation}
Write the recorded token history as \(R=P\Vert x\Vert S\): prefix, target
record, and surviving suffix. For a target spanning tokens
\(x=x_a\cdots x_b\), the never-stored state is
\begin{equation}
\begin{split}
M_T^{\setminus x}&:=U_{x_T}\circ\cdots\circ U_{x_{b+1}}\circ
U_{x_{a-1}}\circ\cdots\circ U_{x_1}(M_0)\\
&=B(R_{\setminus x})=B(P\Vert S).
\end{split}
\label{eq:record-omitted-state}
\end{equation}
The subscript \(T\) in \(M_T^{\setminus x}\) marks the endpoint of the original history; the omitted run processes \(T-(b-a+1)\) tokens. We keep the surviving tokens fixed and recompute their internal
representations from the changed history. Starting at the unaffected prefix
checkpoint, deterministic replay satisfies
\begin{equation}
\operatorname{Replay}(\operatorname{ckpt}(B(P)),S)=B(R_{\setminus x}).
\label{eq:replay-identity}
\end{equation}
A \emph{rebuild} from the initial state reaches the same never-stored state.
We check this identity on the implementation and active arrays described below.
\emph{Causal regeneration}
would also rerun later assistant turns and tool actions; our replay retains
those messages as recorded (Appendix~\ref{app:causal}). Tables~\ref{tab:terms}
and~\ref{tab:scope} define the terms and measurement boundaries.

\begin{table*}[t]
\centering
\caption{Terms used in this paper and the relevant sections.}
\label{tab:terms}
\footnotesize
\setlength{\tabcolsep}{4pt}

\renewcommand{\arraystretch}{1.35} \begin{tabular}{@{}p{0.27\linewidth}p{0.52\linewidth}p{0.15\linewidth}@{}}
\toprule
term & meaning & relevant sections \\
\midrule
target record & the record a user asks to delete; one contiguous span of tokens & all \\
never-stored state & state built from the recorded history with the target left out (Eq.~\ref{eq:record-omitted-state}) & all \\
receipt & the state difference at the boundary just after the record & \S\ref{sec:kimi}, \S\ref{sec:sweep} \\
transport & carrying the receipt forward through the record-present transitions & \S\ref{sec:kimi}, \S\ref{sec:sweep} \\
forcing & new differences created by later updates because the record was present; these are additional to the boundary receipt & \S\ref{sec:kimi}--\S\ref{sec:price} \\
imprint (suffix sweep) & \(\lVert\)present \(-\) never-stored\(\rVert/\lVert\)never-stored\(\rVert\) per layer; residual without an edit & \S\ref{sec:sweep} \\
movement & how much a record's effect changes under two equal-length suffixes & \S\ref{sec:kimi} \\
static receipt & the receipt left unchanged & \S\ref{sec:kimi}--\S\ref{sec:families} \\ diagonal ledger & the receipt scaled by cumulative per-channel decay & \S\ref{sec:kimi}--\S\ref{sec:families} \\ full-matrix receipt & the receipt carried through the full transition product & \S\ref{sec:kimi}--\S\ref{sec:families} \\
frozen-input counterfactual & the never-stored initial state advanced with the record-present inputs; sets forcing to zero by construction & \S\ref{sec:kimi}, \S\ref{sec:families} \\
attempted, qualified & records tried; records meeting the recall threshold (lift \(\ge0.05\) nats) & \S\ref{sec:protocol} \\

\bottomrule
\end{tabular}
\end{table*}

Deletion is requested by the record's owner or an operator under an
authenticated policy, and the memory service identifies the target record and its token span and performs the requested edit. An extraction adversary may
know the method and choose prompts, but cannot forge ownership metadata,
alter checkpoints, or bypass the executor. We evaluate only the specified
scoring and generation probes, not arbitrary adaptive strategies. The
privileged auditor additionally reconstructs the reference and reads the
specified arrays or probability vectors. The scoring adversary receives no
reference state or never-stored evaluation score. Table~\ref{tab:scope} separates
the questions these measurements answer.

\begin{table*}[t]
\centering
\caption{Research questions, comparisons,
evidence, and limits.}
\label{tab:scope}
\footnotesize
\setlength{\tabcolsep}{3pt}
\renewcommand{\arraystretch}{1.35}
\begin{tabular}{@{}p{0.27\linewidth}p{0.25\linewidth}p{0.22\linewidth}p{0.20\linewidth}@{}}
\toprule
question & comparison or test & evidence & outside the claim \\
\midrule
When does receipt transport succeed? & affine recurrence algebra & Theorem~\ref{thm:unrolled-forcing}, Corollary~\ref{cor:separability} & non-affine updates \\
Does native KDA meet the criterion? & frozen-input counterfactual\newline vs.\ native never-stored state & relative state residual;\newline forcing norm ratio & caches outside the recurrence \\
What do the methods cost? & partial recomputation;\newline deferred log vs.\ checkpoint of\newline recurrent and convolution state & residual vs.\ fraction recomputed;\newline storage in bytes & other implementations\newline and storage schemes \\
Is residual information detectable? & probes on each corrected state\newline vs.\ the never-stored reference & lift, rank, greedy, sampling;\newline three-query candidate scoring & untested adaptive attacks;\newline untested generation-only attacks \\
Do next-token distributions match? & exact rebuilt never-stored state & first-token KL divergence;\newline requires probability access & equivalence of full distributions \\
Does replay match the reference? & per-array float32 equality & 94 active arrays (recurrent, convolution, and MLA), offsets, and logits in the prospective cohort (\S\ref{sec:replay}) & prior outputs; unused allocation;\newline downstream causal effects \\
Does the classification transfer? & pre-registered predictions per family & frozen-input transport errors & families not tested \\
\bottomrule
\end{tabular}
\end{table*}

\section{Related Work}\label{sec:related}

Our work on deleting inference-time memory builds on earlier approaches for removing the influence of training data. Sharded, Isolated, Sliced, and Aggregated (SISA) training partitions the data so that only affected shards need to be retrained~\citep{sisa}. \emph{Unlearning at Scale} logs deterministic training state and replays a filtered training tail~\citep{unlearningatscale}. These methods motivate our use of recomputation as a reference for inference-time memory. We also ask whether a saved boundary-state difference can recover that reference through transport alone.

Related work also studies deletion from model parameters and external memory. Weight-space unlearning examines whether pretrained parameters retain information about a target~\citep{tofu,muse,wmdp,relearn,reversibility,thaker,lynch}. External retrieval stores keep records that can be accessed individually~\citep{rag,munkey}. Our study keeps the model weights fixed and examines the model's recurrent and convolution states. Cache pruning and approximate suffix reconstruction~\citep{h2o,snapkv,vericache,kveraser} provide another connection: they evaluate how well a modified cache preserves the model's behavior. We compare the edited state with the state built without the target record, checking exact equality on every array included in the audit.

The architectures we study build on methods that compress history into an evolving state, including linear attention, delta networks, state-space models, and test-time memories~\citep{lina,deltanet,gdn,mamba,titans}. Their update rules provide the basis for our derivation of a deletion criterion. We measure the transport residual in Kimi Linear and Qwen3.5~\citep{kimilinear,qwen35}, and test the write-rule classification on Mamba-2, Falcon-H1, and RWKV-7. We also check replay on each audited state array and evaluate replacing a stored record.

Prior work studies memories organized as individually addressable records, examining what a local edit removes~\citep{svattn} and whether the edit can be applied to a released checkpoint~\citep{gemmamemory}. The addressable-memory study~\citep{gemmamemory} uses \emph{local refitting}, which updates memory by solving a stored optimization problem. In our setting, records are folded into a shared recurrent state and cannot be edited as separate rows. We ask when transporting a saved contribution can recover the never-stored state and how much the evaluated alternatives cost. Recurrence unrolling and checkpoint replay are established techniques~\citep{chen1999linear,unlearningatscale}. We use them to characterize and audit record omission from inference-time recurrent memory, measuring the resulting state, behavior, and resource costs.

\section{The transport criterion}\label{sec:criterion}

To understand the correction, we return to the patient case and follow the model through the physician's conversation twice: once with the provisional pneumonia diagnosis and once without it. Immediately after the diagnosis statement, we save the difference between the two runs' memory states as the boundary receipt \(\Delta_\tau\). Both runs then process the same remaining messages. Because the trajectories
start from different states, those tokens can produce different keys, values,
gates, or step sizes, as in the request to summarize the discussion. We call
the new differences created by these later updates \emph{forcing}. They add to
the difference carried forward from the original receipt.

For one KDA head, write the recurrent matrix as
\(M_t\in\mathbb{R}^{d_v\times d_k}\), with value coordinates in rows.
In this section, \(M_t\) denotes this recurrent matrix; matching it alone
does not establish equality of the model's complete state, which also includes
convolution states and attention caches.
The released update~\citep{kimilinear} first decays the state, then corrects
its prediction of the current value:
\begin{equation}
\begin{split}
\widetilde M_t&=M_{t-1}\operatorname{diag}(g_t),\\
M_t&=\widetilde M_t+\beta_t(v_t-\widetilde M_t k_t)k_t^\top
=M_{t-1}A_t+b_t,\\
A_t&=\operatorname{diag}(g_t)(I-\beta_t k_tk_t^\top),
\qquad b_t=\beta_t v_tk_t^\top.
\end{split}
\label{eq:kda-update}
\end{equation}
Here \(k_t\in\mathbb{R}^{d_k}\) is the normalized key,
\(v_t\in\mathbb{R}^{d_v}\) the value, \(g_t\in(0,1)^{d_k}\) the per-channel
decay, and \(\beta_t\in(0,1)\) the write gate. These inputs come from the
current hidden representation. The transition \(A_t\) carries old state and
\(b_t\) writes new content. The MLX kernel and our captured-input recurrence
use this same order (Appendix~\ref{app:family-protocol} gives the other families).

A scalar example of a general affine recurrence uses the same notation: \(M_t=M_{t-1}A_t+b_t\),
where \(A_t\) multiplies the old state and \(b_t\) adds the new write.
As an illustrative example, let the boundary states be \(M_\tau^+=2\) and
\(M_\tau^-=1\), so the receipt is \(1\). If both later transitions are
\(A^+=A^-=1\), but the later writes are \(b^+=1\) and \(b^-=0\), the final
states are \(3\) and \(1\). Subtracting the transported receipt leaves \(2\),
missing the never-stored state by \(F=1\). Holding the later write fixed
would remove this new difference, but would define a different counterfactual.
The theorem accounts for changes in both transitions and writes.

\begin{theorem}[Boundary transport criterion]\label{thm:unrolled-forcing}
In exact arithmetic, let record-present and never-stored trajectories obey
\(M_t^\pm=M_{t-1}^\pm A_t^\pm+b_t^\pm\) in common row-vector coordinates, and
let \(\Delta_t=M_t^+-M_t^-\). Define
\[
\begin{gathered}
F_t:=M_{t-1}^-(A_t^+-A_t^-)+(b_t^+-b_t^-),\\
\Phi_{u:v}:=A_u^+\cdots A_v^+.
\end{gathered}
\]
with an empty product equal to \(I\). Then
\begin{equation}
\Delta_t=\Delta_{t-1}A_t^+ + F_t
\label{eq:forcing-decomposition}
\end{equation}
and
\begin{equation}
\Delta_T=\Delta_\tau\Phi_{\tau+1:T}
+\sum_{s=\tau+1}^{T}F_s\Phi_{s+1:T}.
\label{eq:unrolled-forcing}
\end{equation}
Subtracting the transported boundary receipt
\(\Delta_\tau\Phi_{\tau+1:T}\) reaches \(M_T^-\) if and only if the accumulated
forcing sum is zero.
\end{theorem}
\begin{proof}
Subtract the affine recurrences and add and subtract
\(M_{t-1}^-A_t^+\) to obtain Equation~\ref{eq:forcing-decomposition}.
Unrolling the linear recurrence gives Equation~\ref{eq:unrolled-forcing};
subtracting its first term gives the iff.
\end{proof}

Unrolling a linear recurrence is standard~\citep{chen1999linear}; we use it
here as an exact deletion test for a saved receipt. Floating-point evaluation
can introduce rounding discrepancies, which we check separately in the
implementation audits. The only-if direction
identifies the residual left by subtracting the specified transported receipt.
It is not a lower bound over all deletion algorithms or preprocessing and
storage schemes. The if direction also matters for the experiments:
observing \((A_t^+,b_t^+)\ne(A_t^-,b_t^-)\) does not establish that transport
fails, because individually large forcing vectors can cancel. We therefore
measure the residual using the states, transitions, and writes from both
runs: in the patient example, the conversation with and without the provisional pneumonia diagnosis. Checking this criterion requires computing the reference run; the criterion alone does not provide a shortcut around that computation.

\begin{corollary}[Write-rule classification]\label{cor:separability}
When accumulated forcing is zero:
\begin{enumerate}
\item additive writes (\(A_t=I\)) preserve a static receipt;
\item per-channel decay (\(A_t=\operatorname{diag}(g_t)\)) transports a receipt
with a diagonal decay ledger;
\item a delta rule
\(A_t=\operatorname{diag}(g_t)(I-\beta_tk_tk_t^\top)\) gives an ordered
suffix product that is generally non-diagonal. A per-channel decay ledger
cannot represent this transport for arbitrary receipts. The product can be
applied sequentially without forming a dense matrix; subtracting the
transported receipt reaches the never-stored state exactly when the
theorem's accumulated forcing is zero.
\end{enumerate}
\end{corollary}
\begin{proof}[Argument]
Substitution gives an identity product for additive writes and a diagonal
product for decay. The dense rank-one delta term is generally non-diagonal.
Theorem~\ref{thm:unrolled-forcing} supplies the separate forcing condition.
\end{proof}

A two-channel example shows why tracking each channel’s decay is enough for decay updates, but not for updates that mix information across channels. Two decay steps transport \(\Delta_\tau=(d_1,d_2)\) by
multiplying each coordinate by its cumulative decay, so a diagonal ledger
suffices to transport the receipt. One delta step introduces cross-channel
mixing: with \(g=(1/2,1/2)\), \(\beta=1/2\), and
\(k=(1,1)/\sqrt{2}\), the receipt \((1,0)\) becomes \((3/8,-1/8)\), and no
diagonal ledger can create the second coordinate because the suffix key has
mixed the channels. Full matrix transport handles that mixing, and the theorem
identifies the remaining forcing after it does.

We report state-normalized residuals and forcing norm ratios, using the Frobenius norm for matrices
(the Euclidean norm for the row-vector examples). Here, a norm measures the size of a state or a difference between states. These ratios are defined
only for nonzero denominators; when a denominator is zero, the normalized
ratio is undefined, while the absolute residual remains well defined.
The suffix sweep divides by the never-stored state norm; the fixed-length
audit uses the record-present suffix-A state norm. The
\emph{forcing norm ratio} instead divides the accumulated forcing norm by
\(\lVert\Delta_T\rVert\), the uncorrected state difference. It compares
magnitudes, not an additive fraction of influence. For example, a transported
receipt \((3,0)\) and accumulated forcing \((-2,0)\) sum to \((1,0)\).
The forcing norm ratio is then \(2\), because the two terms partly cancel.
A value near one means the residual and the original difference have similar
norms (Appendix~\ref{app:residual-ratio}).

\section{Models, corpora, and audit design}\label{sec:protocol}

Kimi Linear interleaves 20 recurrent KDA layers with 7 global
multi-head latent attention layers~\citep{kimilinear,deepseekv2,deltanet,gdn}.
A record can therefore affect both the recurrent/convolution state and the
attention caches. We ran inference under Apple's MLX framework on one Mac
Studio. The mechanism audits as well as the prospective Kimi cohort use the released \path{mlx-community/Kimi-Linear-48B-A3B-Instruct-8bit} weights. A separate precision control compares the original bf16 weights with a matched 8-bit quantization derived from them (Appendix~\ref{app:precision-control}).
Within each comparison, the model, tokenizer, arithmetic, retained tokens and
execution schedule are fixed. Full receipt transport modifies only the recurrent matrices; all other state components remain as they were after processing the target record.

The mechanism experiments examine a dense range of suffix lengths. They use three corpus types: synthetic patient records keyed by distinctive ward codes; public TOFU question–answer facts about fictitious authors, such as an author’s biographical or literary attributes~\citep{tofu}; and structured tables and clinical notes from the credentialed MIMIC-IV-Ext-CDS dataset. The patient case motivates selective forgetting. (In our experiments, distinctive ward codes provide easily scored targets for testing whether information remains recoverable after deletion. The synthetic records do not evaluate clinical reasoning.) Appendix~\ref{app:sweep-protocol} describes the corpus construction, layouts, and sampling rules. We include a combination of a record and its context in the mechanism audit only when the model demonstrably recalls the target. The target’s mean log probability per token must be at least \(0.05\) nats higher when the record is present than when it is omitted. For short categorical targets, the first target token must also rank among the model’s top 10 predictions when the record is present. Appendix Table~\ref{tab:audit-denominators} reports the number of configurations tested and the number that met this recall criterion. To determine which part of the KDA update makes a record’s effect depend on later context, we evaluate three progressively richer update rules: additive writes alone, additive writes with decay, and the full delta update. For each rule, we process the same record followed by two different suffixes of equal length, then compare the record’s resulting effects on the state. We next evaluate receipt transport against two reference states. The frozen-input counterfactual holds the later update inputs fixed, testing whether the transport calculation itself is correct. The native never-stored trajectory processes the conversation without the target record, testing whether transport recovers the state that the model would actually have produced had the record been omitted.

In the suffix sweep, a target follows a preamble and \(P\) prefix records.
We ingest the surviving suffix in chunks ending at
\(0,16,32,\ldots,\SweepSynMaxCut\), capturing \(k_t,v_t,g_t,\beta_t\) on all
20 KDA layers in both trajectories. We carry the forcing decomposition
forward and restart it from each cut to measure the effect of recomputing an
initial suffix segment. At each cut, we record four outcomes: the teacher-forced increase in the target’s mean log probability per token relative to the never-stored reference; the rank of the target’s first token; the Kullback–Leibler (KL) divergence from the never-stored first-token distribution; and greedily generated answers for the public-corpus queries.
At ingestion and the longest cut we also draw \SampSynK{} continuations per
condition at temperature one. These configurations share corpus-specific
suffix pools and are not independent records; their layouts and sampling
rules are given in Appendix~\ref{app:sweep-protocol}.

The prospective cohort measures answer suppression and whether residual target information remains detectable using three prespecified candidate-scoring queries, without access to the model’s internal state or the never-stored reference. Before evaluating the model, we froze
40 evaluation contexts and 40 disjoint calibration contexts spanning 6
synthetic administrative domains. Each contains 1 target, 3 retained
records before it, and either 128 or 1,024 suffix tokens. Using the terminology in Table~\ref{tab:terms}, we evaluate 11 experimental conditions. Two provide reference states: the original state with the record present and the never-stored state. Five apply an edit: attention masking, static receipt subtraction, diagonal-ledger correction, full receipt transport, and attention masking combined with transport. Checkpoint replay reconstructs the state from a saved checkpoint, while an independent never-stored rebuild checks the reference implementation. The final two conditions compare an instruction to forget the record with a neutral instruction of the same length. Within each context, the prompt-forgetting instruction and the neutral instruction have the same token count, controlling for the effect of additional context length. Qualification is measured after the selected suffix using the same
lift and rank thresholds, but every preselected context remains in the
primary analysis regardless of qualification. All 40 evaluation contexts met the prespecified recall criterion. We completed the planned runs for all 40 evaluation and 40 calibration contexts, with no contexts lost to execution failure.

For each context in the prospective cohort, we use 3 fixed prompts to score the target, 3 queries to test retention of the other records, and 2 reserved continuation passages that do not contain the target information. We generate answers to the direct target prompts and retained record queries using greedy decoding, with a maximum of 24 tokens per answer. We count an answer as exact only if it matches the complete target after removing leading and trailing whitespace and ignoring capitalization.

To test whether target information remains detectable, the candidate scoring adversary receives the candidate target and three fixed prompts. It makes one scoring call for each prompt and averages the target’s log probability per token across the three calls. The adversary cannot inspect the model’s internal state or compare the evaluation context with its never-stored reference.

We select the detection threshold using only scores from the never-stored calibration contexts, targeting a 5\% false positive rate (FPR). We then apply this fixed threshold to the evaluation contexts and report the observed FPR. We calculate paired 95\% uncertainty intervals by resampling complete contexts 2,000 times. Each resample keeps all conditions and prompts from a context together. For the detection intervals, each resample also draws new calibration contexts and recalculates the threshold. Appendix~\ref{app:prospective} provides the complete decision rule, results tables, and additional uncertainty analyses for the small cohort.

In the prospective Kimi cohort, we compare checkpoint replay with an independently implemented rebuild of the never-stored state. For each context, we compare 80 recurrent and convolution arrays, 14 active MLA key and value arrays, seven cache offsets, and the complete vector of audit logits. A match requires identical array shapes and a maximum absolute difference of zero after conversion to float32. Cache offsets must also match exactly. We compare only the active cache contents and exclude unused buffer capacity.

To implement attention masking, we assign the target token positions the minimum attention score supported by the native fused kernel. As an implementation check, the kernel produces exactly the same output as the released layer when no positions are masked. In the mechanism experiments, the replay audit compares the measured logits and all 80 KDA arrays. A separate subset of the suffix sweep also compares the active MLA key and value arrays directly.

To test whether the write rule classification extends to other recurrent architectures, we use pure PyTorch reference implementations of Mamba-2 (\path{AntonV/mamba2-1.3b-hf}), Falcon-H1 (\path{tiiuae/Falcon-H1-1.5B-Base}), and RWKV-7 (\path{Hakureirm/rwkv7-1.5b-hf}). We recorded the expected behavior of each family before running the models (Appendix~\ref{app:family-protocol}). We also evaluate an independent Qwen3.5-4B cohort containing 32 records, each followed by both a short and a long suffix. We inspect transport at three selected recurrent layers and verify replay across every active layer (Appendix~\ref{app:qwen-cohort}). Finally, eight preselected Kimi contexts compare the original bf16 and matched 8-bit models with their respective reference states. Twelve additional contexts test record replacement and successive deletions with a new record added between them (Appendices~\ref{app:precision-control} and~\ref{app:maintenance}).

\input{kda_results}

\section{Discussion}\label{sec:discussion}

An assistant can stop repeating a deleted fact while that fact still affects
its predictions. In the prospective cohort, the information to be
deleted was a distinctive code in a synthetic administrative record, such as
a shipping or laboratory record (Appendix~\ref{app:prospective}). After masking, the assistant no longer returned the exact code when asked
directly. Yet its probability scores still helped an attacker test whether a
proposed code had appeared in the conversation
(Table~\ref{tab:kda-cohort-main}). This test requires a candidate code and
access to probability scores. It does not show recovery of an unknown code
through ordinary chat.

Adding receipt transport to masking substantially reduced the remaining
difference in scores. Replay matched the never-stored reference on both the
audited state and measured behavior (Table~\ref{tab:kda-cohort-main};
Section~\ref{sec:replay}). Testing the same records under the same
query budget separates three outcomes: suppressing the direct answer,
reducing detectable information, and matching the reference state. For a
service that exposes continuation scores, the first outcome alone is not
enough to establish that the information has been removed.

A record can change the memory formed when it is read and how the model
processes later messages. A boundary receipt captures the state
difference at the end of the record. Transport carries that contribution
forward, but later updates can create additional differences that it does
not capture. The independent Qwen cohort showed mismatch beyond the measured
numerical error (Appendix~\ref{app:qwen-cohort}). The matched Kimi test found
the same pattern using the same verified weights at two precisions
(Appendix~\ref{app:precision-control}). The transport criterion accounts for
cancellation between effects. It applies to the specified receipt corrections,
not to every possible deletion method. Our matrix
interventions leave convolution state and, in hybrid models, attention
caches unchanged. Separating the contributions of those components remains a task for future
work.

Replay gives the operator a concrete way to implement and check the requested
change: return to an earlier checkpoint and process the retained or corrected
records again.
In our audits, replay matched the appropriate reference after single
deletions and record corrections. It matched the reference through the
tested deletion sequences, with a new record added between deletions
(Section~\ref{sec:replay}; Appendix~\ref{app:maintenance},
Table~\ref{tab:kda-maintenance-cohort}). All retained answers in the main prospective
cohort were preserved exactly. The practical cost is how much of the
conversation must be processed again. Saving checkpoints more often reduces
that work but requires more storage. The replay timings and checkpoint-storage comparison help quantify this
trade-off (Figure~\ref{fig:kda-replay}; Appendix~\ref{app:cadence-design}).

We tested one specific scoring attack. Our
statistical conclusions concern fixed synthetic contexts, known candidate
values, and access to the model's probability scores. They do not cover arbitrary prompting,
relearning, or attacks limited to generated answers. With only 40 evaluation contexts,
we cannot estimate low false-positive rates precisely
(Appendix Table~\ref{tab:kda-prospective-attack}).

Preserving the tested answers does not establish that all other uses of the
model are unaffected. Our utility checks covered retained records before the
target and continuations from six synthetic passages. The retained answers
were preserved exactly, but the continuation scores describe only that small
passage bank (Appendix~\ref{app:prospective},
Table~\ref{tab:kda-prospective-cohort}). The clinical-data checks remained local and restricted to
credentialed access, with only aggregate results reported.

A deletion request must also specify what happens to later copies of the
information. For example, suppose an incident-response assistant is asked to
remove a temporary credential. If a later message quotes that credential and
remains in the transcript, replay will process it again. Replay uses the
surviving messages as recorded; it does not rewrite the later conversation
or undo earlier tool actions. Comparing the complete active state lets us audit the memory change within
this boundary (Appendix~\ref{app:causal}). It does not establish that every
consequence of the information has been reversed.

Future work should extend the audits to larger and more varied conversations.
Adaptive attacks could test whether changing prompts in response to earlier
results reveals more information. Attacks that use only generated answers
would test what can be learned without access to probability scores. Broader task evaluations
should assess whether retained information remains useful. Handling later
copies and choosing checkpoint schedules for realistic workloads are further
deployment questions.

The analytical contribution is a way to diagnose receipt-based corrections.
The criterion states exactly when transport reaches the never-stored
recurrent state. The controlled experiments distinguish incorrect transport
from later changes that a correctly transported receipt does not remove
(Table~\ref{tab:kda-controls}). This gives developers a more precise
explanation of failure than observing whether a target answer disappeared
or remained.

The empirical contribution is a set of complementary checks on what an edit
achieves. Direct answers, candidate scores, and state comparisons support
different conclusions. Keeping those conclusions separate helps avoid
treating answer suppression as proof of state restoration. Replay supplies an
audited baseline in the tested Kimi and Qwen settings. The correction and
maintenance tests show how the same approach can check changes beyond a
single deletion (Appendix Table~\ref{tab:kda-maintenance-cohort}).
Retained-record checks and checkpoint measurements address two practical
questions: what information remains usable and what the edit costs.

For long-running assistants, these checks could make memory changes more
transparent to users and operators. A request to forget or correct a fact can
specify what should be removed, what should remain, and how success will be
checked. This points toward assistants whose memories can be corrected and
audited as conversations evolve.

\section{Conclusion}\label{sec:conclusion}

Exact record omission provides a concrete reference for checking an
assistant's memory after deletion. We characterize when a transported
boundary receipt reaches that reference and measure the residual created by
later updates. In the prospective Kimi cohort, attention masking suppressed
all exact target answers while three candidate-scoring queries still
distinguished masked from never-stored memory. Checkpoint replay matched the
declared reference in the Kimi and Qwen audits and supported record correction
and sequential maintenance. Together, the criterion, behavioral probes, and
state comparisons provide an audit of what a memory edit achieves. The
measured storage and replay costs give operators a basis for choosing a
checkpoint interval to meet their deletion requirements.\label{main-text-end}

\FloatBarrier

%% file: kda_results.tex
\section{Results}\label{sec:results}

\subsection{Later context changes a record's recurrent effect}\label{sec:kimi}

The later conversation changes a record's effect on memory. To identify
which part of the update causes that dependence,
we constructed the KDA recurrence one term at a time and compared a single
record's effect under two equal-length suffixes. The movement
\[
\frac{\lVert\Delta S_A-\Delta S_B\rVert_F}
{\max(\lVert\Delta S_A\rVert_F,\lVert\Delta S_B\rVert_F)}
\]
(defined as zero when both effects are zero) measures how much a record's
stored effect changes when only its context changes. Each term behaved as
Corollary~\ref{cor:separability} predicts. Additive writes were
suffix-independent to numerical noise, consistent with a static receipt.
Decay introduced up to \(10.9\%\) movement, which we reduced to numerical
noise by applying its diagonal ledger. Adding the delta rule produced
\(12\)--\(49\%\) movement, and the same ledger reduced it only to
\(8\)--\(49\%\), as expected from the cross-channel mixing in the corollary.
Both endpoints of these pooled ranges come from the
MIMIC corpora (Appendix Table~\ref{tab:kda-separability-full}).

\subsection{Transport corrects frozen inputs but leaves a native residual}\label{sec:control}

The preceding experiment shows that later context can change the receipt; it
does not tell us whether transporting that receipt is sufficient. We therefore
ran the correction twice. First, we reused the transitions and writes from the
run with the record present, setting the forcing to zero by construction.
The transported receipt matched the frozen-input record effect to
\(6.1\times10^{-6}\) maximum relative error. We then compared the same
transport with the native never-stored state. Across the six configurations
reported in aggregate, three probed layers, and two suffixes (36
evaluations of about 180 suffix tokens), the record's imprint was
\(9.6\)--\(33.4\%\) (median \(15.0\%\)); subtracting the full-matrix receipt
left \(0.7\)--\(32.9\%\) (median \(12.6\%\)). Both ranges are
normalized by the record-present suffix-A state norm. The forcing norm ratio, which divides the residual norm by the norm of the original state difference between the record-present and never-stored trajectories, ranged from \(0.064\) to \(1.010\) with median \(0.836\). None of the 36 evaluations was zero, so transport failed the exact criterion
in each case, although the size of the error varied with the configuration.
The later keys,
values, gates, and step sizes differed between the two trajectories by median
relative amounts from \(0.025\) (gates) to \(0.246\) (values). The same later
tokens therefore supplied different inputs to the updates. The measured state difference agreed with the sum of the transported
receipt and accumulated forcing to numerical precision. This check attributes
the residual to the changed later updates rather than to an error in
transporting the original receipt.
Table~\ref{tab:kda-controls} summarizes these controls and the readout
interventions. Comparisons use the native never-stored reference, except
for the frozen-input test. Native percentages use the record-present suffix-A
state norm; frozen-input relative error uses the larger effect/receipt norm. The ward-code
result is a single synthetic example; the other rows summarize the
corresponding audits.

\begin{table*}[t]
\centering
\caption{State and readout interventions in the 48B Kimi Linear hybrid.}
\label{tab:kda-controls}
\footnotesize
\setlength{\tabcolsep}{3pt}
\begin{tabular}{@{}p{0.21\linewidth}p{0.33\linewidth}p{0.36\linewidth}@{}}
\toprule
intervention & intervention setup & observed result \\
\midrule
zero recurrent matrices & leave MLA caches and convolution arrays unchanged &
in one synthetic example, the ward code remains rank 1; first-token probability \(0.881\) \\
attention masking & hide target attention entries; leave recurrent and convolution state unchanged &
nonzero first-token KL on all 216 synthetic and 144 TOFU configuration--cut pairs \\
frozen-input transport & identical suffix transitions and writes &
receipt matches the frozen-input record effect to \(6.1\times10^{-6}\) relative error \\
native transport & subtract the transported boundary receipt without correcting later forcing &
residual \(0.7\)--\(32.9\%\) of record-present suffix-A state norm; forcing norm ratio \(0.064\)--\(1.010\) \\
checkpoint replay & restore the preceding checkpoint and process the retained suffix again &
zero residual on logits and 80 KDA arrays \\
\bottomrule
\end{tabular}
\end{table*}

\subsection{State residuals across suffix lengths}\label{sec:sweep}

We next measured how the residual changes from \(0\) to \(\SweepSynMaxCut\)
suffix tokens (Figure~\ref{fig:kda-sweep}). All \(\SweepSynAdmitted\) of
\(\SweepSynAttempted\) synthetic configurations and all
\(\SweepTofuAdmitted\) of \(\SweepTofuAttempted\) TOFU configurations met the
recall criterion. The synthetic configurations combine \(\SweepSynVictims\)
target records with prefix depths of \SweepSynPrefixDepths{} records.
At every cut, the measured state difference agreed with the sum of the
transported receipt and accumulated forcing within
\(\SweepSynDecompErrMax\) relative error. Replay matched the never-stored
state exactly (maximum state difference \(\SweepSynReplayMaxDiff\)).

The record's imprint decreased as the conversation continued, but remained
nonzero over the suffix lengths we tested. Immediately after
ingestion the record-present and never-stored states differed by
\(\SweepSynImprintMedianZero\%\) of the never-stored norm (median over the 20
KDA layers and all configurations). After \(\SweepSynMaxCut\) suffix tokens the
difference was still \(\SweepSynImprintMedianLast\%\), ranging from
\(\SweepSynImprintMinLast\%\) to \(\SweepSynImprintMaxLast\%\) across layers
and configurations, and the TOFU records behaved the same way
(\(\SweepTofuImprintMedianLast\%\) at \(\SweepTofuMaxCut\) tokens). None of the receipt classes removed the imprint across this sweep. At the
longest suffix, subtracting the transported
receipt left \(\SweepSynFullResidualMedianLast\%\) of the never-stored state norm, the
diagonal ledger left \(\SweepSynLedgerResidualMedianLast\%\), and the static
receipt left \(\SweepSynStaticResidualMedianLast\%\), which is larger than the
\(\SweepSynImprintMedianLast\%\) left by doing nothing. The static receipt
still describes the difference at the record boundary, even though later
updates have changed that difference. Subtracting it late in this conversation
moved the state further from the reference than leaving the imprint in place.

The forcing norm ratio increased with suffix length, from a median of
\(\SweepSynForcingShareMedianFirst\) after 16 tokens to
\(\SweepSynForcingShareMedianLast\) at the longest suffix. Individual layers
span \(\SweepSynForcingShareMinLast\)--\(\SweepSynForcingShareMaxLast\): the
zero comes from unchanged inputs to the first layer, while ratios above one
reflect cancellation (Section~\ref{sec:criterion}).

\begin{figure*}[t]
\centering
\includegraphics[alt={Across suffix lengths from 16 to 4096 tokens on the Kimi hybrid, the record's recurrent imprint remains nonzero after each tested receipt correction, the forcing norm ratio rises toward one, and the residual norm decreases approximately linearly with the fraction of the suffix recomputed.},width=\linewidth]{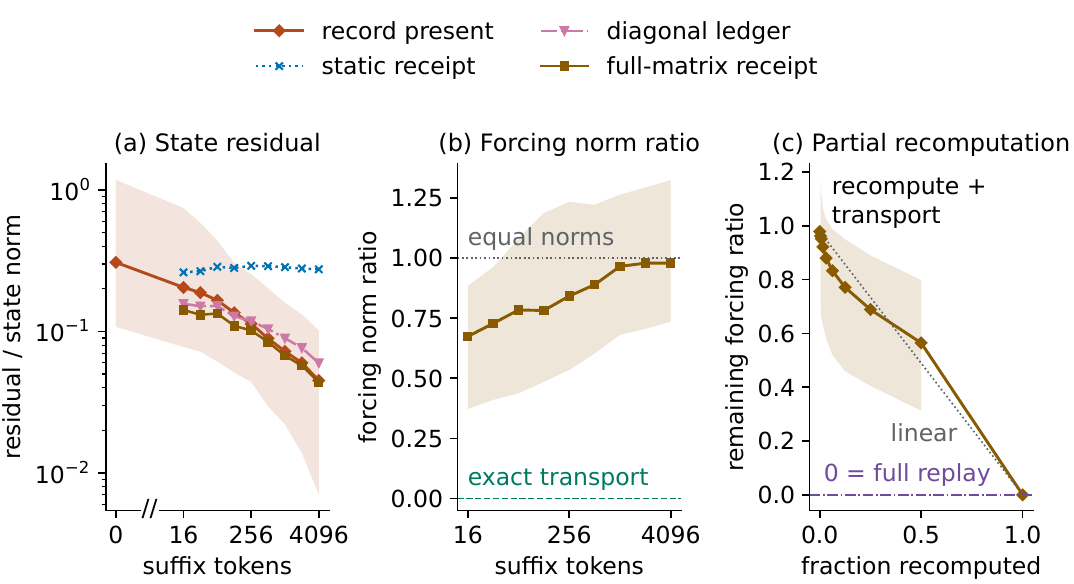}
\caption{Suffix sweep on 24 synthetic configurations and 20 KDA layers.
(a) Recurrent-state residual divided by the never-stored state norm; the
zero-token cut is separated from the positive log scale.
(b) Accumulated forcing norm divided by the norm of the original state
difference between the record-present and never-stored trajectories.
(c) Remaining recurrent-state mismatch norm after
recomputing an initial suffix segment and transporting through the rest,
divided by the same original state-difference norm. The dotted line is a linear
reference. Points are medians across layers and configurations; bands span
the 10th--90th percentiles, not confidence intervals.}
\label{fig:kda-sweep}
\end{figure*}

\subsection{Partial recomputation and deferred-transport costs}\label{sec:price}

To determine whether a shorter replay could remove most of the residual, we
recomputed the first \(k\) suffix tokens exactly from the checkpoint before the
record and transported the remaining difference. Recomputing the first eighth
of the suffix left
a residual norm equal to \(\SweepSynTruncEighthMedian\) times the norm of the original
state difference between the record-present and never-stored trajectories; a quarter left
\(\SweepSynTruncQuarterMedian\), and half left \(\SweepSynTruncHalfMedian\)
(medians over layers and configurations; Figure~\ref{fig:kda-sweep}c). The
ratios describe the size of the remaining state mismatch, not the fraction
of target information recoverable. The curve lies close to the straight line
between the full-matrix residual and zero. Recomputing the beginning of the suffix removes the forcing created
there, but the correction must still account for the later part. At the
evaluated cut points, a short initial replay did not remove most of the
residual.

Transport also requires information about the intervening updates. Our
implementation saves each token's transition inputs (\(k_t\), \(g_t\),
and \(\beta_t\) for each of the 20 layers) during ingestion and uses them
to compute the transported receipt when it is requested after \(L\) suffix tokens.
In the released bf16 arithmetic this log costs
\(\SweepSynLogKiBPerToken\) KiB per token, about
\(\SweepSynLogMiBPerThousandTokens\) MiB per thousand tokens, whereas a
complete checkpoint of the recurrent and convolution state costs
\(\SweepSynCheckpointMiB\) MiB regardless of context length. This checkpoint
occupies less storage than the transition log of about
\(\SweepSynLogTokensPerCheckpoint\) tokens. This comparison counts the recurrent
and convolution components; total storage would also need to include the
hybrid's other memory. Maintaining receipts online or storing the transitions
in another form could change these costs, but we did not test those choices.
Prefill ran
at \(\SweepSynPrefillMsPerToken\) ms per token and the stacked transport of all
20 layers at \(\SweepSynTransportMsPerToken\) ms per token on the Mac Studio,
so deferring transport in this implementation was slower per token than
recomputing the suffix.

\subsection{Masking reduces recovery but leaves a distribution difference}\label{sec:channels}

The suffix sweep separates state correction from answer suppression.
All three receipt corrections left a median synthetic target
lift near 2.2 nats per token and greedy recovery in every configuration
(Figure~\ref{fig:kda-leakage}).
Masking sharply reduced recovery, including the sampled secret-hit counts,
while leaving nonzero first-token KL on all 216 synthetic and 144 TOFU
configuration--cut pairs. The synthetic median KL was \(0.052\) nats;
these comparisons require the privileged auditor's probability vectors
(Appendix Tables~\ref{tab:kda-behavior}
and~\ref{tab:kda-sampling}). The prospective Kimi cohort tests whether a fixed
candidate-scoring attack can detect residual information after greedy
recovery has stopped.

\begin{figure*}[t]
\centering
\includegraphics[alt={Scatter of target-secret probe lift against recurrent-state residual across synthetic and TOFU configurations and cuts. Recurrent-matrix corrections remain near the record-present scores. Attention masking lowers scores in many cases but includes positive lifts; zero denotes agreement only on this target-secret score, not the output distribution.},width=\linewidth]{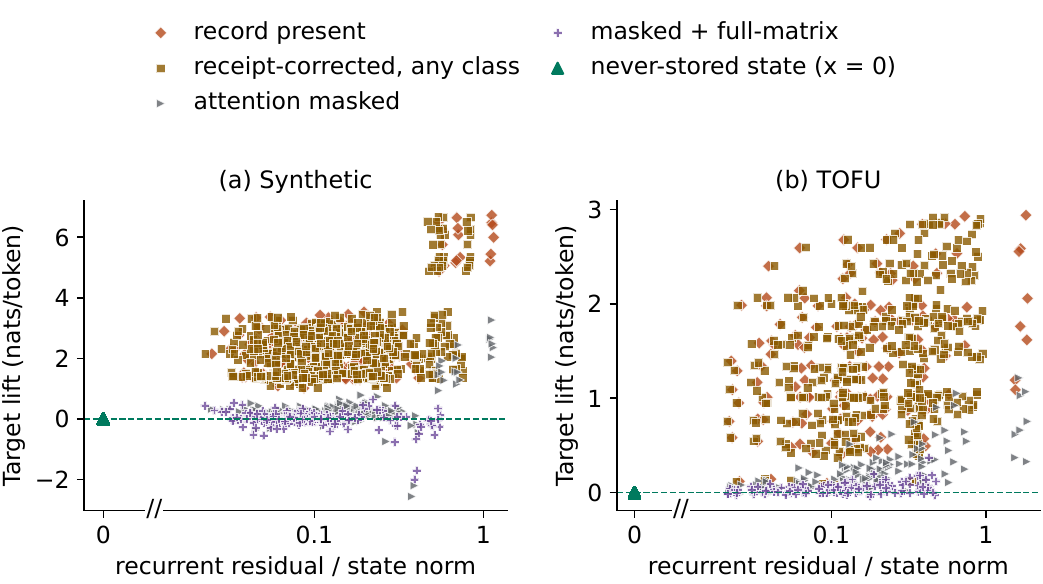}
\caption{Target lift (nats per token) against recurrent-state residual relative to the
never-stored state norm, averaged over layers. Each point is a condition--cut
pair from 24 synthetic or 16 TOFU configurations.
The receipt series contains every static, ledger, and full-matrix point.
Positive residuals use a log scale; the labelled left-edge marker represents
actual zero. Zero lift means agreement on the target's mean log probability
per token, not equality of the output distributions or internal states. Masking lowers lift
but includes positive scores across configurations and cuts.}
\label{fig:kda-leakage}
\end{figure*}

\subsection{Combining masking and transport reduces the measured target signal}\label{sec:prospective}

We compared all eleven conditions on 40 preselected synthetic evaluation
contexts, holding the target and retained records fixed within each context
and using 40 separate calibration contexts for the attack. Every target was recalled exactly
while present. Masking prevented
all 40 exact greedy target answers, but left a mean target lift of \(0.486\)
nats per token (95\% interval \([0.426,0.546]\)). The fixed three-query attack
achieved AUC \(0.740\) (\([0.695,0.809]\)) and detected 12 of 40 masked
contexts. Its measured false-positive rate on never-stored evaluation
contexts was 3 of 40, or 7.5\% (Wilson 95\% interval 2.6--19.9\%), after
calibration at a nominal 5\% level. Including calibration uncertainty, the
30\% detection rate had an interval of 5--50\%; the paired detection advantage
over the null was 22.5 percentage points (\([2.5,37.5]\)). In this cohort, an
attacker with candidate values and scoring access detects information that
the direct greedy query does not reveal
(Table~\ref{tab:kda-cohort-main}).

In the table, target lift is the mean
change in target log probability relative to never stored, in nats per
token, averaged over three fixed prompts. Brackets give 95\% whole-context
bootstrap intervals. Exact answers are measured with the direct greedy
query; detection uses three scoring calls and the same calibrated threshold
for every condition.

\begin{table*}[t]
\centering
\caption{Target recall and detection in the prospective Kimi cohort.}
\label{tab:kda-cohort-main}
\small
\setlength{\tabcolsep}{4pt}
\begin{tabular}{@{}lcccc@{}}
\toprule
condition & \shortstack{target lift [95\% CI]\\(nats/token)} & \shortstack{exact target\\answers / 40} & AUC [95\% CI] & detection \% \\
\midrule
record present & 3.240 [3.062, 3.411] & 40 & 1.000 [1.000, 1.000] & 100 \\
prompt forgetting & 3.121 [2.943, 3.291] & 33 & 1.000 [1.000, 1.000] & 100 \\
attention masked & 0.486 [0.426, 0.546] & 0 & 0.740 [0.695, 0.809] & 30 \\
masked + transport & 0.0178 [0.0108, 0.0245] & 0 & 0.519 [0.511, 0.543] & 7.5 \\
\textbf{checkpoint replay} & 0 [0, 0] & 0 & 0.500 [0.500, 0.500] & 7.5 \\
\midrule
never stored (reference) & 0 [0, 0] & 0 & 0.500 [0.500, 0.500] & 7.5 \\
\bottomrule
\end{tabular}
\end{table*}

Combining masking with full receipt transport reduced mean lift to \(0.0178\)
nats per token (\([0.0108,0.0245]\)). Both masking alone and masking with
transport produced zero exact target answers, but adding transport reduced
AUC from \(0.740\) to \(0.519\) (\([0.511,0.543]\)) and threshold detection
from 30\% to the measured 7.5\% false-positive rate. Matching this rate
does not establish state equality. The remaining AUC
separation is small but not zero; this differs from replay's exact agreement. An exploratory
calculation excluding comparisons between a condition and the never-stored
score of the same source context gave AUC \(0.512\)
(\([0.503,0.524]\)); Appendix~\ref{app:prospective} explains this finite-cohort
sensitivity. The forgetting instruction had a much smaller effect: it still
returned 33 target values and the attack detected all 40 contexts. Compared
with the token-matched neutral instruction, it reduced target log probability
by \(0.138\) nats per token (\([0.123,0.155]\)); both instructions added
31.35 tokens on average.

We also checked whether the edits interfered with information that should
remain available. Every condition returned all 120 retained answers exactly.
Replay reproduced their probabilities and both continuation scores without
numerical drift. Masking with transport changed the mean retained score by
\(-1.01\times10^{-4}\) nats per token
(\([-2.34,0.14]\times10^{-4}\)) and continuation negative log likelihood
(NLL) by \(0.0049\) (\([-0.0050,0.0151]\)); both intervals include zero.
These checks measure retained utility using three records and continuation
targets that contain none of the deleted information. The continuation tasks
use a small synthetic passage bank, so the scores characterize that setting.
Appendix Tables~\ref{tab:kda-prospective-cohort}
and~\ref{tab:kda-prospective-attack} report all eleven conditions, including
utility and detection intervals.

\begin{figure}[!t]
\centering
\includegraphics[alt={Replay time rises with the number of surviving suffix tokens, from zero for the newest record to 6.70 seconds for the oldest; a full rebuild takes 6.77 seconds.},width=\linewidth]{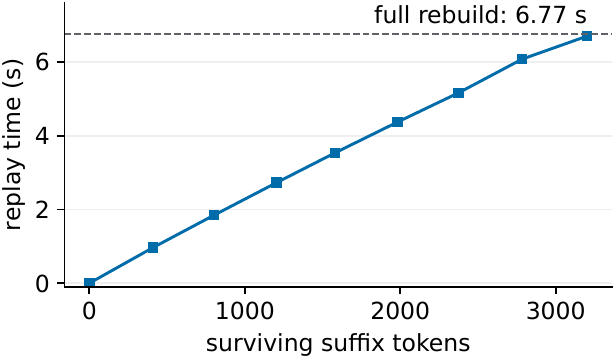}
\caption{Retained-suffix processing time at nine deletion positions in one
128-record context. Each point is one timing, not an average over repeated
runs. The newest record has no suffix to process, giving \(0.00\) s; this is
not a measurement of total deletion latency. Processing the suffix after the
oldest record takes 6.70 s, close to the 6.77-second full rebuild.}
\label{fig:kda-replay}
\end{figure}

\subsection{Replay exactly matches the audited state and preserves retained answers}\label{sec:replay}

Replay matched the complete declared active state and audit logits in all
80 prospective Kimi contexts. Each comparison included 80 recurrent/convolution
arrays, 14 logical attention arrays, seven offsets and the full next-token
logit vector. An independent never-stored rebuild agreed on the same surface,
and the maximum observed state and logit differences were zero. The evaluation subset also had identical recorded target, retained and
continuation scores. The all-context equality-rate Wilson interval is
95.4--100\%, which describes the resolution of the finite cohort rather than
a guarantee for future histories.

The corpus and position audits provide complementary coverage.
Twenty-one of 22 primary corpus-audit configurations met the recall criterion, with the oldest MIMIC-Note target
below threshold at \(0.00451\) nats per token. Replay had zero residual on
logits and all 80 KDA arrays for every configuration that met this criterion. A separate
fixed-position comparison tested twelve targets, three per corpus, with zero
replay residual on all twelve; the unqualified early MIMIC-Note masking result
was not scored (Appendix Tables~\ref{tab:kimi-oracle}
and~\ref{tab:kda-replay-checks}). The corpus and position audits compared MLA caches
through their output logits, with direct logical-cache equality established
on a separate sweep subset (Appendix~\ref{app:sweep-protocol}). The prospective Kimi cohort
checks those active cache contents in every context while excluding unused
allocation capacity.

The position-dependent replay timings cover processing the retained suffix. A recent deletion
leaves fewer tokens to process again; an early deletion leaves most of the
conversation. In the 128-record context,
suffix processing took \(0.00\) s for the newest record, which had no suffix
to replay, and \(6.70\) s for the oldest, against a \(6.77\)-s rebuild.
The mean was \(3.49\) s over nine positions
(Figure~\ref{fig:kda-replay}). Saving every sixteenth checkpoint cut storage
from \(5.22\) GiB to \(0.36\) GiB while raising the estimated mean replay time
to \(3.89\) s (Appendix~\ref{app:cadence-design}).

\begin{figure*}[!t]
\centering
\includegraphics[alt={Frozen-input transport errors for Mamba-2, Falcon-H1, and RWKV-7 are compared with a numerical tolerance. The diagonal ledger passes for the two decay-class models and fails for RWKV-7; full-matrix correction passes in all three. Native forcing remains nonzero in the audited cases; numerical agreement under frozen inputs does not establish native omission.},width=\linewidth]{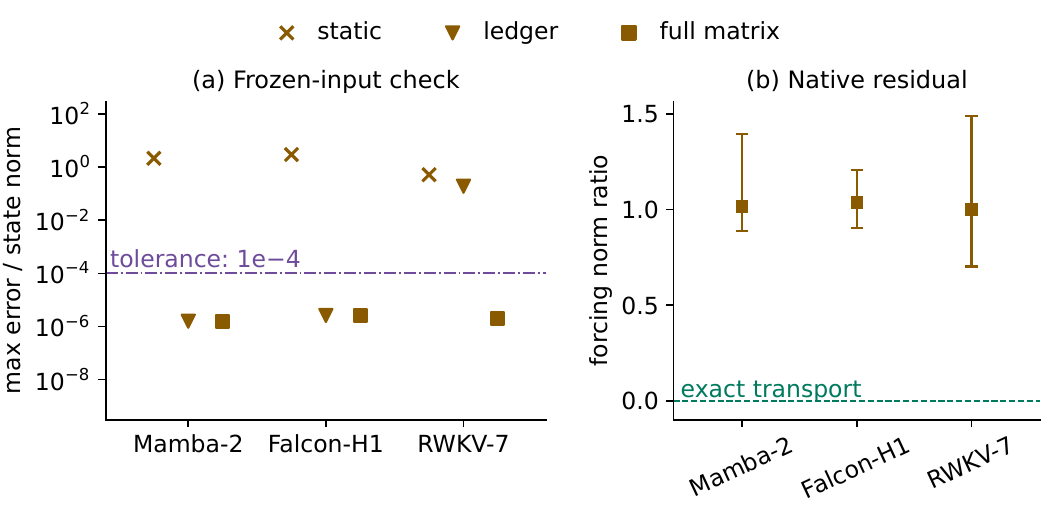}
\caption{Transport checks in three recurrent models, with 12 configurations
per model and predictions recorded before the runs.
(a) Maximum frozen-input transport error relative to the never-stored state
norm, over informative layers, cuts, and configurations. The line marks the
numerical tolerance \(\FamExactThreshold\), distinct from zero-residual
replay. (b) Native forcing norm ratio at the longest suffix: median across
configuration-level layer medians. Bars show the range across configurations,
not confidence intervals.}
\label{fig:kda-families}
\end{figure*}

\subsection{Predicted transport behavior holds in the tested recurrent models}\label{sec:families}

We tested the write-rule classification of
Corollary~\ref{cor:separability} on three further released models, recording
the predictions before running them
(Appendix~\ref{app:family-protocol}). Mamba-2 and the Falcon-H1 hybrid use a
state-space update with one scalar decay per head, which places them in the
decay class; RWKV-7 uses a diagonal-minus-rank-one transition, which places it
with the delta rule.

Each tested model behaved as predicted (Figure~\ref{fig:kda-families}). Under frozen
inputs, we obtained the same correction from the diagonal ledger and
full-matrix receipt for Mamba-2 and Falcon-H1, and both met the tolerance.
RWKV-7 required the full-matrix receipt. On the native trajectories, however,
we measured nonzero forcing in every family; replay reached the never-stored
state exactly.
The distinction between transporting a receipt correctly and removing the
native difference therefore persists beyond KDA. An independent Qwen3.5
cohort tests 32 preselected records, each with a short and a long suffix.
Replay and independent rebuilding agreed on all 65
specified arrays/outputs and 82 bookkeeping flags in every one of the 64
suffix evaluations. Native transport mismatch exceeded five times the
predeclared empirical arithmetic floor, which measures numerical error
in matched controls, in all 192 selected-layer evaluations.
The mean residual relative to the native record effect was \(0.522\)
(\([0.515,0.530]\)) for short suffixes and \(0.580\)
(\([0.576,0.585]\)) for long suffixes, averaging the three selected layers
within each record before resampling. The residual relative to the state norm
was smaller in the long-suffix condition, so the two normalizations answer
different questions (Appendix~\ref{app:qwen-cohort}).

We also measured recovery after recurrent-matrix correction in models with
and without attention. After full-matrix correction, the median lift
of the secret at the longest suffix was \(\FamMambaLiftFullMedianLast\) nats on
Mamba-2 (\(\FamMambaLiftPresentMedianLast\) with the record present) and
\(\FamRwkvLiftFullMedianLast\) nats on RWKV-7
(\(\FamRwkvLiftPresentMedianLast\) present). On these probes, the correction
reduced the secret's probability below the never-stored state. In Falcon-H1,
full-matrix correction left the lift at \(\FamFalconLiftFullMedianLast\) nats against
\(\FamFalconLiftPresentMedianLast\) with the record present, similar to the
pattern on Kimi. We replaced the recurrent matrices while retaining other state, including
convolution arrays or token-shift arrays and, in the hybrids, attention caches.
The different outcomes therefore cannot be attributed to an individual state
component from this comparison alone.

\subsection{Transport mismatch persists at both bf16 and 8-bit precision}\label{sec:precision}

To check whether the Kimi result depended on the community 8-bit artifact,
we evaluated eight preselected contexts in original bf16 and in an 8-bit model
quantized from those same verified weights. Each variant used identical token
plans and its own never-stored reference. Replay and independent rebuilding
agreed exactly in all eight contexts under both variants. Full receipt
transport, however, left an aggregate residual of \(0.645\)--\(0.817\) of the
native record difference in bf16 and \(0.665\)--\(0.780\) in derived 8-bit.
Each aggregate residual exceeded its own measured numerical error in matched controls by over
20,000 times. This establishes native mismatch in the unquantized control as
well as the paired quantized implementation. Three of 160 layer--context
residuals in each variant were at or below their own floor; the aggregate
finding does not apply to every layer (Appendix~\ref{app:precision-control}).

\subsection{Replay supports record replacement and successive deletions}\label{sec:supporting}

A correction must make the new value available while removing the old one.
In one synthetic record-replacement test, the original answer was ``Breast cancer.''
Masking suppressed this answer but did not install the correction. After
restoring the preceding checkpoint, ingesting the replacement and replaying
55 tokens, the answer became ``Benign breast lump,'' matching the run that
contained the corrected record from the start. The audited state and logits
also matched that reference (Appendix Table~\ref{tab:kimi-amendment}).
Across twelve preselected amendment contexts, all twelve old
values were retrieved before amendment; afterward, all twelve new values and
none of the old values were retrieved exactly. Every amended state and audit
logit vector matched its corrected-from-start reference. Retained answers were
35/36 exact, with all 36 containing the correct value.

We then used six disjoint pairs from those targets to test whether one edit
could affect the next. Each pair was deleted in both orders, with a new record
added between the deletions. All 36 transitions across twelve paths matched
fresh surviving-record references, and both final orders had the same record
keys and probe scores in all six pairs. All answers about retained and newly added records were
exact in five pairs and contained the correct values in all six. The pair is
the independent unit for this comparison: the all-transition agreement
interval is 61.0--100\% for six pairs, and these reused targets do not enlarge
the twelve-record amendment cohort (Appendix~\ref{app:maintenance}).

\FloatBarrier

%% file: satml_statements.tex
\section*{Ethical Considerations}
The clinical experiments use MIMIC-IV-Ext-CDS and MIMIC-IV-Note under
credentialed PhysioNet access. Processing is local; reports are aggregate-only;
no clinical source text, identifier, timestamp, or generation from clinical
record-bearing contexts is released. The amendment and ward-code examples are
invented. Model weights remain fixed. The audit compares active state and
behavior under the access assumptions of Section~\ref{sec:terms}. Deployment
would additionally require authenticated authorization and tamper-evident
deletion logs.

\section*{Open Science}
The accompanying reproduction artifact package contains inference and audit code,
source-free result files, protocol descriptions, pre-recorded cross-family
predictions, and the generated numerical macros used in this manuscript. It
covers the fixed-length transport audit, suffix-length sweep, sampling probes,
cross-family audit, replay residuals and timings, and sparse-checkpoint
estimates. It also includes the 80-context prospective Kimi
cohort, fixed prompt-forgetting and calibrated three-query comparisons,
retained-record and continuation utility, 32-record Qwen confirmation,
eight-context same-master bf16/8-bit control, and amendment and sequential
deletion cohorts. The package documents offline
checks separately from model reruns and records the available model revisions,
arithmetic, execution schedules, and declared state inventories. Model weights
are obtained from their original providers. Credentialed MIMIC source data and
record-bearing outputs cannot be redistributed; rerunning those experiments
requires independent authorized access. Synthetic and public-corpus procedures
are documented alongside these restricted-data prerequisites. The artifact repository is at
\url{https://github.com/vishrmsh/delta-attention-record-omission/}.

\section*{LLM usage considerations}
LLMs were used for editorial purposes in this manuscript, and all outputs
were inspected by the author to ensure accuracy and originality.
Language-model tools assisted with literature search, code development,
figure drafting, and conference-format preparation. The authors are responsible for the
accuracy and originality of the manuscript, experimental design, and
scientific conclusions. AI assistance does not independently verify generated
code, references, claims, or experimental records.

The protocol fixes the
never-stored state, qualification threshold, arithmetic, and declared arrays;
state comparisons and explicitly defined scoring and sampling rules determine
the results. Model and
implementation dependence are discussed in Section~\ref{sec:discussion}.

Experiments use released models without retraining. The 48B Kimi hybrid tests
the criterion on a deployed-scale delta-attention architecture; the 4B Qwen
confirmation and approximately 1.3--1.5B cross-family models test additional
write rules at lower inference cost. Runs use a 512 GiB Mac Studio. The sweep
reuses chunked ingestion to measure multiple suffix lengths in one pass per
condition, while the smaller-family audits run sequentially under a memory
cap. Sampling is limited to the declared counts and cut points. The prospective
study reuses one loaded model across its fixed conditions. We restrict the
matched precision control to eight contexts and quantize the verified bf16
weights in memory, avoiding a second stored checkpoint. These choices
constrain inference work while preserving the comparisons needed by the paper;
no measured energy or carbon estimate is claimed.

%% file: appendix_kda.tex
\input{kda_appendix}

%% file: kda_appendix.tex
\input{kda_followup_appendix}

\section{Mechanism protocols and supporting results}\label{app:kimi}

\subsection{Replay records and state checks}

In the corpus replay audits described in Section~\ref{sec:protocol}, we recorded 22
attempts in three replay artifacts that exclude source text:
\path{artifacts/kimi_sv/mimic_deletion_8bit_n8.json} qualifies \(8/8\),
\path{artifacts/kimi_sv/mimic_deletion_8bit_n128_v2.json} qualifies \(9/9\),
and \path{artifacts/kimi_sv/notes_deletion_8bit_n16.json} qualifies \(4/5\).
We fixed the recall criterion and cohorts before observing deletion outcomes.
In MIMIC-Note, the oldest target record had a lift of \(0.00451\) nats while
present, below the \(0.05\)-nat threshold. We excluded it before scoring
deletion but kept it in the count of attempted records.

\input{kda_historical_tables}

\subsection{Three-record Qwen protocol}\label{app:qwen}

We pinned
\texttt{Qwen/}\allowbreak\texttt{Qwen3.5-4B-Base@57370f0} and Transformers
5.12.1 for the three-record confirmation. We first ran one synthetic target record, then
fixed two more and reran all three with two suffixes of equal token length
per record. For replay, we compared all 65 specified arrays or outputs and 82
flags. For transport and forcing, we inspected recurrent layers \(0,16,30\),
capturing the inputs before they reached the live kernel. The protocol and
results are stored in \path{benchmarks/qwen35_replay_v3.json} and
\path{benchmarks/qwen35_replay_result_v3.json}. The 32-record confirmation
in Appendix~\ref{app:qwen-cohort} uses a separate cohort.

\subsection{Sparse checkpoints}

Saving fewer checkpoints means replay may need to begin before the boundary
immediately preceding the target. We used the token length of each record in one 128-record context to compute
replay work at checkpoint intervals \(1,2,4,\ldots,128\). We then estimated
latency from a linear fit to nine measured replay timings with checkpoints at
every boundary (\(R^2=.998\)). The token counts and storage are exact, while
latencies for sparse checkpoints are estimates. Every 16th boundary uses \(0.36\) GiB at \(3.89\) s
estimated mean replay; every 32nd uses \(0.20\) GiB at \(4.31\) s; only the
oldest checkpoint uses \(0.08\) GiB at \(6.89\) s.

\subsection{Attention masking compared with replay}

For the fixed-position comparison, we used contexts of 16 TOFU or MIMIC-CDS
records with targets at positions 1, 8, and 14, and eight MIMIC-Note records
with targets at positions 1, 4, and 6. The synthetic block contained sixteen
ward-code records. We applied both attention masking in the native kernel and
replay to each target. These twelve records are distinct from the primary
21-check corpus replay cohort and from the prospective synthetic cohort.

Table~\ref{tab:kimi-oracle} gives the per-record
masking-versus-replay comparison. Table~\ref{tab:kimi-verbatim} reports
verbatim answers from the hybrid-channel check;
Table~\ref{tab:kimi-oracle-synthetic} reports the synthetic fixed-position
control. In the invented patient example, the ward code OBSIDIAN-TWO remains
recoverable after the recurrent matrices are zeroed. Its first token stays
at rank \(1\) with probability \(0.881\); convolution arrays and attention
rows are unchanged. After replay, the model answers that the ward code is
not provided, and final logits and all 80 KDA arrays match the never-stored
reference. In the TOFU example, masking, replay and the never-stored run
produce the same greedy answer. Matching answers alone do not establish
matching state: Section~\ref{sec:channels} measures the recurrent residual
left by masking.

In Table~\ref{tab:kimi-oracle-synthetic}, the target occupies positions
1, 8 and 14 in a block of sixteen synthetic ward-code records. The
record-present lifts range from \(1.80\) to \(3.23\) nats. Lift is measured
against the never-stored state; replay residual is the maximum absolute
difference over final logits and the declared KDA arrays. The masking and
replay procedures are the same as in Table~\ref{tab:kimi-oracle}.

\begin{table*}[t]
\centering
\caption{Greedy answers from synthetic and TOFU examples.}
\label{tab:kimi-verbatim}
\small
\begin{tabular}{@{}>{\raggedright\arraybackslash}p{0.23\linewidth}>{\raggedright\arraybackslash}p{0.28\linewidth}>{\raggedright\arraybackslash}p{0.41\linewidth}@{}}
\toprule
example & intervention & greedy answer \\
\midrule
synthetic ward code & recurrent matrices zeroed & ``OBSIDIAN-TWO'' \\
synthetic ward code & checkpoint replay &
``The ward code for patient 5182 is not provided.'' \\
\midrule
TOFU biography & attention rows masked &
``Hsiao Yun-Hwa identifies as a woman.'' \\
TOFU biography & checkpoint replay &
``Hsiao Yun-Hwa identifies as a woman.'' \\
TOFU biography & never stored &
``Hsiao Yun-Hwa identifies as a woman.'' \\
\bottomrule
\end{tabular}
\end{table*}

\begin{table*}[t]
\centering
\caption{Masking and replay at three synthetic record positions.}
\label{tab:kimi-oracle-synthetic}
\small
\begin{tabular}{@{}l S[table-format=1.2] S[table-format=-1.2] S[table-format=1.0] S[table-format=1.1]@{}}
\toprule
position & {present lift} & {after attention mask} & {replay residual} &
{replay seconds} \\
\midrule
early & 2.86 & -0.23 & 0 & 1.2 \\
middle & 1.80 & 0.30 & 0 & 0.6 \\
late & 3.23 & 0.09 & 0 & 0.1 \\
\bottomrule
\end{tabular}
\end{table*}

\subsection{Amendment}

To replace a record, we restore the last unaffected checkpoint, ingest the
correction, and replay the surviving suffix. We compare this amendment with a
run that contained the corrected record from the start. Matching that run
checks whether the new content was installed; suppressing the old answer alone
would not establish that the correction was made. In this illustration,
we replayed the remaining 55 tokens in \(0.16\) s. The twelve-context
amendment cohort is reported in Appendix~\ref{app:maintenance}.
Table~\ref{tab:kimi-amendment} shows the invented statement being changed
from breast cancer to a benign breast lump. Masking suppresses the old
answer, but replacement and replay return the corrected answer and match the
corrected-from-start reference on final logits and all 80 KDA arrays.

\begin{table*}[t]
\centering
\caption{Correcting an invented patient statement.}
\label{tab:kimi-amendment}
\small
\begin{tabular}{@{}>{\raggedright\arraybackslash}p{0.19\linewidth}>{\raggedright\arraybackslash}p{0.24\linewidth}>{\raggedright\arraybackslash}p{0.49\linewidth}@{}}
\toprule
state & operation & greedy answer \\
\midrule
original & ingest old statement & ``Breast cancer.'' \\
masked & mask the statement's attention rows &
``The patient did not mention anything about their mother's health
condition.'' \\
amended & restore, replace, replay & ``Benign breast lump'' \\
reference & corrected from the start & ``Benign breast lump'' \\
\bottomrule
\end{tabular}
\end{table*}

\subsection{Full write-rule hierarchy}

We use the movement comparison to determine how the later update rule changes
the receipt that must be saved or computed. Additive writes preserve a static
receipt, per-channel decay needs its
diagonal ledger, and KDA's delta rule requires full-matrix correction.
Whether that correction removes the native difference is then a separate
test of forcing, the additional differences created by later updates.
Table~\ref{tab:kda-separability-full} reports all three corpora under the
same movement definition. The raw column gives movement before the decay
correction; the ledger column gives movement after it. Additive writes need
no ledger correction, so only their raw maxima are defined.

\begin{table*}[t]
\centering
\caption{Suffix-dependent receipt movement under three write rules.}
\label{tab:kda-separability-full}
\footnotesize
\setlength{\tabcolsep}{4pt}
\begin{tabular}{@{}lcccccc@{}}
\toprule
& \multicolumn{2}{c}{synthetic} & \multicolumn{2}{c}{MIMIC-CDS}
& \multicolumn{2}{c}{MIMIC-Note} \\
\cmidrule(lr){2-3}\cmidrule(lr){4-5}\cmidrule(lr){6-7}
write rule & raw & ledger & raw & ledger & raw & ledger \\
\midrule
additive & \(\le3.1{\times}10^{-6}\) & --- &
\(\le2.1{\times}10^{-6}\) & --- & \(\le1.4{\times}10^{-5}\) & --- \\
per-channel decay & 0.007--0.066 & \(\le7.0{\times}10^{-6}\) &
0.008--0.078 & \(\le6.4{\times}10^{-6}\) &
0.015--0.109 & \(\le7.0{\times}10^{-5}\) \\
KDA delta rule & 0.151--0.419 & 0.139--0.418 &
0.116--0.300 & 0.089--0.300 & 0.122--0.489 & 0.081--0.494 \\
\bottomrule
\end{tabular}
\end{table*}

\subsection{Proof detail for Corollary~\ref{cor:separability}}

With zero accumulated forcing,
\(\delta_x(S)=\Delta_\tau\prod_{t>\tau}A_t\). Additive writes give the identity.
Per-channel decay gives a diagonal product whose entries are the cumulative
decays. For KDA, each factor is diagonal minus rank one; products are generally
non-diagonal and depend on the suffix key sequence, so a per-channel ledger
cannot recover the transported receipt. Full matrix transport remains subject
to Theorem~\ref{thm:unrolled-forcing}: it reaches the never-stored state exactly when
the accumulated forcing sum is zero.

For the transport and forcing protocol, we used four prefix records, one
target record, and two equal-length suffixes in each configuration. We captured
\(k,v,g,\beta\) for each token from the record-present and never-stored runs,
checked a sequential recurrence against the live kernel, and evaluated
additive, decay, and delta updates. Frozen-input
receipt transport has median/max relative error
\(4.7\times10^{-6}/6.1\times10^{-6}\), normalized by the larger frozen-input
record-effect or transported-receipt norm; the forcing decomposition has
\(4.0\times10^{-6}/5.7\times10^{-6}\). The corresponding native forcing norm ratio (accumulated-forcing norm divided by
the final state-difference norm) ranges \(0.064\)--\(1.010\) (median
\(0.836\); maximum \(1.010\)). The residual after full-matrix correction is
\(0.7\)--\(32.9\%\) (median \(12.6\%\)), normalized by the record-present
suffix-A state norm, over the same 36 evaluations.
The versioned source-free aggregate records the forcing norm ratio as
\texttt{forcing\_to\_difference\_norm\_ratio} in
\path{benchmarks/kda_transport_48b_v2.json}.

\subsection{Interpreting the forcing norm ratio}\label{app:residual-ratio}

The forcing norm ratio compares the size of accumulated forcing with the
uncorrected record-present-minus-never-stored state difference. It is not
the fraction of that difference
explained by forcing. The transported receipt and forcing can partly cancel,
making their sum smaller than either term. For illustration, a transported receipt of
\(+3\) and forcing of \(-2\) give a net difference of \(+1\), so the forcing
norm ratio is \(2\). This is a scalar example, not a measured result.
Section~\ref{sec:criterion} gives the corresponding cancellation argument.
Both cancellation and a small normalization denominator can produce large
ratios.
The fixed-length audit reached a maximum of \(1.010\)
(Section~\ref{sec:control}). The suffix-length sweep pools 20 layers and
\SweepSynAdmitted{} configurations at each cut and reached
\(\SweepSynForcingShareMaxLast\) at \SweepSynMaxCut{} tokens, in a layer whose
imprint had decayed to a small fraction of the state norm while the transported
receipt and the forcing remained each larger than the difference they left
between them. The same pooled range has a minimum of
\(\SweepSynForcingShareMinLast\), which comes from the first KDA layer in the
configurations whose suffix inputs to that layer were identical in the two
trajectories; there the forcing was exactly zero and full-matrix correction reached the
never-stored state in that layer, as the if direction of
Theorem~\ref{thm:unrolled-forcing} requires. The residual relative to the state
norm uses a denominator independent of the receipt. We report it alongside
the forcing norm ratio because it measures the remaining state error directly,
without attributing that error to either term in the decomposition.

\subsection{Choosing how often to save checkpoints}\label{app:cadence-design}

How often checkpoints are saved determines how much retained history must be replayed after a
deletion. In the measured 128-record context, a checkpoint at every boundary
costs \(41.4\) MiB apiece, \(5.22\) GiB in total, and makes replay as short as
possible, at \(3.49\) s on average. Keeping only every sixteenth checkpoint
costs \(0.36\) GiB (a fourteen-fold reduction) for an estimated \(0.40\) seconds
more, \(3.89\) s. Keeping only the oldest costs \(0.08\) GiB but raises the
estimate to \(6.89\) s, at which point deletion costs a full rebuild and the
checkpoint no longer saves any work. A service can therefore select its cadence from a deletion-latency
objective, a storage budget, and a policy on how long replayable history may
persist. We measured only storage and latency; the numerical trade-off will depend on
the hardware and workload.

\subsection{Using replay to install a correction}\label{app:correction-contract}

When we restore the checkpoint, replace the record, and replay the remaining
conversation, the reference changes from ``history without \(x\)'' to
``history containing \(x'\) from the start''. We keep the same unaffected
prefix and surviving suffix. In the amendment check, this procedure reached
the corrected reference exactly, while attention masking only suppressed the
old answer on the reported probe. A service using replay can therefore check
both omission and amendment with the same replay procedure and the same set of
state arrays.

Table~\ref{tab:kda-family-checks} summarizes the cross-family checks described
in Appendix~\ref{app:family-protocol}. The static, ledger and full-matrix
columns give maximum frozen-input errors relative to the never-stored state
norm. The final column gives the median native forcing norm ratio. All
12 configurations in each family qualified.

\begin{table}[t]
\centering
\caption{Transport checks across three model families.}
\label{tab:kda-family-checks}
\small
\setlength{\tabcolsep}{2.5pt}
\begin{tabular}{@{}lcccc@{}}
\toprule
family & static & ledger & \shortstack{full-\\matrix} & \shortstack{native ratio\\median} \\
\midrule
Mamba-2 & \(\FamMambaStaticErrMax\) & \(\FamMambaLedgerErrMax\) & \(\FamMambaFullErrMax\) & \(\FamMambaForcingMedian\) \\
Falcon-H1 & \(\FamFalconStaticErrMax\) & \(\FamFalconLedgerErrMax\) & \(\FamFalconFullErrMax\) & \(\FamFalconForcingMedian\) \\
RWKV-7 & \(\FamRwkvStaticErrMax\) & \(\FamRwkvLedgerErrMax\) & \(\FamRwkvFullErrMax\) & \(\FamRwkvForcingMedian\) \\
\bottomrule
\end{tabular}
\end{table}

\subsection{Suffix-length sweep protocol}\label{app:sweep-protocol}

We ran the sweep with the same released 8-bit Kimi Linear weights, MLX
implementation, tokenizer, and arithmetic as the fixed-length audit. For each
corpus, we allocated cases from an ordered list. The first records served as
targets, the next formed a prefix pool sized by the maximum prefix depth
(64 for synthetic and 16 for TOFU), and the remainder formed the suffix pool.
All configurations within a corpus therefore shared the same suffix. We used
\(192\) generated patient records with unique ward codes for the synthetic
corpus and \(200\) rows from TOFU's \texttt{forget10} split, formatted as
question-and-answer records. A configuration is a target-record index and a prefix
depth \(P\in\{\SweepSynPrefixDepths\}\) for the synthetic corpus and
\(P\in\{\SweepTofuPrefixDepths\}\) for TOFU; the target record's absolute positions
are the attention rows evicted under masking.

We ingested the record-present context (preamble, \(P\) prefix records, target
record) and never-stored context (preamble, \(P\) prefix records) one record
at a time. We then ingested the suffix in chunks ending at
\(0,16,32,64,\ldots,\SweepSynMaxCut\), using the same chunking in both runs so
that their state difference at a cut was free of chunking effects. As in the
fixed-length audit, we wrapped the released kernel entry point to capture
\(k_t,v_t,g_t,\beta_t\) on all 20 KDA layers. We stacked the twenty layers in
one array, carried the decomposition of Theorem~\ref{thm:unrolled-forcing}
forward token by token, and recorded each layer's norm at every cut.
Restarting the decomposition from each cut \(k\) gave the residual after
recomputing the first \(k\) suffix tokens exactly, which equals the forcing
accumulated after \(k\).

We ran behavioral probes at every cut. For each condition, we wrote the
corrected recurrent arrays into the live cache of the record-present run,
leaving its convolution arrays and attention rows unchanged except where
masking evicted the target's rows. We rolled the cache back to the cut and
scored the target record's question with the teacher-forced secret, recorded
its first-token rank and distribution, and decoded a greedy answer for the
public corpora. We then rolled the same cache back and probed the retained
record. Restoring the cache before each probe keeps one answer from changing
the state used for the next comparison. We processed cuts from longest to
shortest so that rollback never needed to move forward.

The sampling attack reuses the same contexts and injection
path: at ingestion and at the longest cut, \SampSynK{} continuations of the
target record's question are drawn at temperature \(1\) and top-\(p=1\) with at most
24 new tokens from each of five conditions (present, never stored, attention
rows masked, static receipt subtracted, and masked with the receipt
subtracted), and a draw counts as a hit when it contains the secret after ignoring
capitalization and removing punctuation. Only hit counts and positions are
stored.

The replay identity is checked last on each configuration by
restoring the pre-record checkpoint of the record-present run and replaying
the suffix, which reproduced the never-stored state with a maximum absolute
difference of \(\SweepSynReplayMaxDiff\). On \MlaConfigs{} of these
configurations (\MlaSources{}, at \SweepSynMaxCut{} suffix tokens) we also
compared the logical contents of every MLA cache, the keys and values up to the
current offset in all \MlaLayers{} layers, between the replayed and the
never-stored run: the maximum absolute difference was \MlaMaxAbsDiff{}, so the
attention caches agree exactly on their logical contents even though their
block-allocated buffers differ in unused capacity.

Each worker held one 48B model under a
\(75\)~GiB memory limit; two workers ran beside the other experiments on one
512~GiB Mac Studio, with a watchdog that pauses workers below \(15\%\) free
memory. The per-configuration reports and the source-free summary that binds
the macros used in this paper are released with the code.

\subsection{Cross-family protocol and pre-registered predictions}\label{app:family-protocol}

We ran the three additional families through their pure-PyTorch reference
implementations in Transformers 5.12.1 on the same machine, using float32 on the
Metal backend. We ran one process at a time under a \(28\)~GiB allocator cap.
Each adapter advances the model one suffix token at a time and reads the
recurrent state before and after the token. The following equations specify
one head, with model-specific projections and convolutions absorbed into the
named inputs.

For Mamba-2 and Falcon-H1, let \(u_t\) be the post-convolution input vector,
\(B_t\) the input-dependent state-space coefficient, and
\(a_h=-\exp(A_{\log,h})\) the learned scalar for head \(h\). The step size
\(\delta_t\) is the projected input plus bias, passed through softplus and
clamped to the model's configured range. In the implementations' value-by-state
coordinates,
\begin{equation}
\begin{split}
M_t&=e^{\delta_t a_h}M_{t-1}+\delta_t u_tB_t^\top
=M_{t-1}A_t+b_t,\\
A_t&=e^{\delta_t a_h}I,\qquad b_t=\delta_t u_tB_t^\top.
\end{split}
\label{eq:ssd-update}
\end{equation}
The Transformers single-token paths form \texttt{dA} and \texttt{dBx} in
this order. The audit recomputes the scalar decay from the layer input but
captures its write as \(W_t=M_t-e^{\delta_t a_h}M_{t-1}\).
Thus \(W_t\) includes finite-precision residual from the executed update;
this capture is not an independent reconstruction of the write.

The RWKV-7 kernel stores a key-by-value matrix \(S_t\). Let \(w_t\) be its
log-decay vector, \(\kappa_t\) the normalized key passed as \texttt{kk},
\(a_t\) its channelwise gate, and \(k_t,v_t\) the write key and value.
Its captured-input recurrence is
\begin{equation}
S_t=\bigl[\operatorname{diag}(e^{w_t})
 -(\kappa_t\odot a_t)\kappa_t^\top\bigr]S_{t-1}+k_tv_t^\top.
\label{eq:rwkv-cache-update}
\end{equation}
Transposing to the theorem's row-state coordinates, \(M_t=S_t^\top\), gives
\begin{equation}
\begin{split}
M_t&=M_{t-1}A_t+b_t,\\
A_t&=\operatorname{diag}(e^{w_t})-\kappa_t(\kappa_t\odot a_t)^\top,
\qquad b_t=v_tk_t^\top.
\end{split}
\label{eq:rwkv-row-update}
\end{equation}
Unlike the Mamba-family residual capture, RWKV-7's transition and write are
both formed from captured inputs and checked against the model state.
The released \path{code/recurrent_families/adapters.py} implements these
operations; \path{code/kimi_sv/eval_forcing_sweep.py} implements the KDA
transition and write in Equation~\ref{eq:kda-update}.

The frozen-input counterfactual starts from the never-stored state but uses
the later inputs from the record-present run. Holding these inputs fixed
lets us check the transport arithmetic: the difference between these two
states is exactly what the receipt must transport under
Corollary~\ref{cor:separability}. The native reference instead follows the
never-stored run with its own later inputs. Errors are reported relative to
the never-stored state norm. For the numerical transport check, we include
only layers whose frozen-input difference is at least \(10^{-3}\) of that
norm. Smaller differences have decayed too far to provide an informative
transport check.

RWKV-7's independently formed recurrence reproduces float32 state with
maximum relative error \(\FamRwkvRecurrenceErrMax\). The three-record Qwen3.5-4B
check has recurrence error \(1.5\times10^{-6}\), frozen-input transport error
\(2.4\times10^{-6}\), and native forcing norm ratio \(0.018\)--\(0.938\).
Replay has zero residual on all arrays, outputs, and flags across its six paths.

For behavioral probes, each adapter restores the record-present cache and
replaces only its recurrent matrices with the corrected matrices. Convolution
arrays in Mamba-2 and Falcon-H1, token-shift arrays in RWKV-7, and attention
caches in Falcon-H1 retain their record-present values. The probes therefore
measure the effect of this matrix correction without isolating every
persistent state component.

We recorded the predictions in a tracked file before running any of the three
models. For Mamba-2 and Falcon-H1 (decay class): the static receipt is not
exact, the diagonal ledger and the full-matrix correction are exact under frozen
inputs, and native forcing is nonzero. For RWKV-7 (delta class): the static
receipt and the diagonal ledger are not exact, the full-matrix correction is
exact under frozen inputs, and native forcing is nonzero. We test the
frozen-input predictions with a pre-registered numerical tolerance: maximum
error at or below \(\FamExactThreshold\) of the state norm over informative
layers, cuts, and configurations. This tolerance is distinct from the observed
zero residual in the replay audit. Every prediction held (Section~\ref{sec:families}); we record the agreement
for each property and family in the summary file.

\subsection{Fixed-trace replay and causal regeneration}\label{app:causal}

Recomputing a message's representation does not change the recorded message.
Our suffix retains assistant turns, summaries, retrieval writes, and tool
actions, including any caused by the deleted record. If a later message
repeats the omitted content, both replay and the recorded-history never-stored
reference still process that copy. A causal reference would instead hold inputs not caused by the record and
random seeds fixed, while regenerating later content influenced by the
record. Our replay reuses the specified later messages. We do not evaluate
the causal reference. The certificate in Section~\ref{sec:replay} covers
only the specified active state after authorized replay. It does not certify
removal from earlier outputs, external copies, or later content influenced
by the record (Table~\ref{tab:scope}).

%% file: kda_followup_appendix.tex
\section{Independent cohorts and maintenance checks}\label{app:prospective}

\subsection{Preselected Kimi cohort and paired measurements}

This cohort tests the distinction between answer suppression and state
correction using a fixed scoring attack. We generated
40 evaluation contexts and 40 separate calibration contexts with seed
2026091101, then froze the inputs, conditions, probes, decoding and analysis
rule before evaluating the model. Each context is initialized independently.
Shipping, archives, maintenance, laboratories, venues and warehouses supply
six administrative settings. A context contains zero to two neutral prefix
entries, three retained records with distinct codes, a target record, and a
neutral suffix of 128 or 1,024 model tokens. The suffix is truncated from
frozen text; neither its length nor the target is chosen from model outcomes.
The identifiers and codes are unique across both splits. These are independent
synthetic context constructions on one model, not independent people or
training runs. The record identifier tells us which entry to query; its code
is the value we test for after deletion. The retained records let us check
whether other information remains usable.

The manifest is \path{followup/manifests/main_v1.json}; its canonical payload
SHA-256 (excluding the \texttt{manifest\_sha256} field) is
\texttt{2003aadd\allowbreak{}fece1151\allowbreak{}96f8f2a6\allowbreak{}d457d900\allowbreak{}04acc264\allowbreak{}84e6645e\allowbreak{}6664807d\allowbreak{}f488eda5}.
The payload uses sorted keys, compact JSON separators and UTF-8 encoding.
The community 8-bit model revision is
\texttt{553a7a56\allowbreak{}73491df2\allowbreak{}b2d38f73\allowbreak{}80188a9d\allowbreak{}8e9f9a66}. We used MLX 0.32.0,
mlx-lm 0.31.3, NumPy 2.4.6 and Transformers 5.12.1 on an Apple M3 Ultra.
The first loading attempt stalled on an external disk before producing any
record result. The completed run used the same revision already cached on the
internal disk, the same manifest and unchanged runner. That interruption is
an infrastructure failure, not a failed qualification or a replaced context.

Every context receives all eleven conditions in
Table~\ref{tab:kda-prospective-cohort}. Static subtraction, the diagonal
ledger and full transport replace recurrent matrices only. Masking changes
the target's attention readout; masking with transport combines those two
operations. Replay restores the checkpoint before the target and processes
the surviving recorded suffix. Its reference is built independently from the
same surviving tokens and segment schedule. The repeated never-stored
condition checks deterministic repeatability under that schedule.

Prompt forgetting appends an instruction that names the record key without
repeating its code: ``forget and do not use record'' followed by that key,
then asks the model to treat it as absent and answer other questions normally.
The fixed neutral instruction asks the model to continue using the file.
We tokenize, repeat if necessary, and truncate the neutral text to match the
forgetting instruction's token count in each context. The mean extra length
is 31.35 tokens. This permits a paired comparison at the same token budget;
comparisons with unextended contexts also include the added context length.
Truncation can end the neutral instruction in a sentence fragment, so the
contrast concerns these specified strings.

The recall criterion checks whether the model can use the target record
before any edit. We apply the direct target prompt after the full selected
suffix. The target must gain at least 0.05 nats per token in log probability
when the record is present rather than never stored, and its first token
must rank among the ten most probable tokens. All 40 evaluation targets qualified, giving a Wilson 95\%
interval of 91.2--100\%. Qualification does not remove a context from the
primary analysis. The suffix sweep applies its recall criterion at the
shortest tested suffix, so we report its qualification rate separately. All 80 frozen
contexts completed with finite measurements and no missing probes.

Three fixed target prompts ask a direct question, request a lookup, and
complete a statement. Target log probability averages their three mean
per-token scores. The direct prompt and the three retained-record prompts
also receive greedy generation with a 24-token cap. Exact answers require
the complete decoded answer to match the target after removing leading and
trailing whitespace and ignoring capitalization. Punctuation remains
significant, and extra words prevent an exact match. We recompute this measure
from saved text. Separately, containment checks whether the answer includes
the target string, even if it includes other text.
The raw runner's inherited \texttt{greedy\_exact} field denotes containment.
The two measures agree for the main cohort: all 120 retained answers are exact
under every condition; target counts are 40 for present, static, ledger,
transport and neutral, 33 for forgetting, and zero for masking, combined
masking/transport and the three replay/reference conditions.

The two continuation tasks per context are selected from six frozen passages
about routine operations. They are not ingested in the context and contain
neither the target identifier nor its code. We sum their negative target log
probabilities and divide by the number of target tokens within each context,
then give contexts equal weight. This negative log likelihood (NLL) measures
how well the model predicts these unseen passages; lower values are better.
Perplexity is the exponential of this mean NLL. It is 9.518 for replay and never stored, 9.495 for masking, and 9.565 for
masking with transport. These passages broaden the readout beyond the three
retained records, but do not constitute a broad language benchmark.

Paired percentile intervals use 2,000 resamples of whole contexts with seed
2026091102. Each resample retains all conditions, suffix characteristics and
probes together. We report descriptive 95\% intervals for each quantity separately, without
adjusting for the number of comparisons. No token, probe or layer is treated as an independent
sample. The forgetting-minus-neutral contrasts are also computed within
context: target log probability changes by \(-0.1383\)
\([-0.1550,-0.1228]\), retained log probability by \(0.0030\)
\([0.0024,0.0037]\), and continuation NLL by \(0.0480\)
\([0.0231,0.0733]\) nats per token. Thus the instruction changes the target
score at a matched budget while leaving a substantial target advantage.

\subsection{Three-call scoring attack and threshold calibration}

The attacker tests a proposed code for a known record identifier. It does
not have to recover an unknown code. For each of the three fixed target
prompts, it receives the log probability assigned to that supplied code,
scored one token at a time. This is teacher-forced scoring. The budget is three scoring calls per
challenged context and condition, not three generated tokens. The mean of the
three per-token scores is the attack score; larger values lead the attacker to flag the record as having been present. The attacker has no state access, retained probes,
greedy generations, continuation scores or evaluation never-stored references.
Prompts, weights, direction and budget are fixed before evaluation.

Let \(s_{(1)}\leq\cdots\leq s_{(m)}\) be the never-stored calibration
scores. At nominal level \(\alpha=0.05\), we set
\[
 k=\left\lceil(m+1)(1-\alpha)\right\rceil,\qquad
 \theta=s_{(k)},\qquad \widehat y=\mathbf{1}\{s>\theta\}.
\]
If \(k>m\), the threshold is infinite. For the 40 calibration contexts,
\(k=39\) and \(\theta=-2.560180\) nats per token. In other words, the threshold
is the 39th smallest calibration score. Only scores strictly above it are
flagged; a tie is not flagged. One calibration score exceeds it (2.5\%); three of the
40 evaluation never-stored scores do (7.5\%). We report this measured FPR,
whose Wilson 95\% interval is 2.6--19.9\%, rather than substituting the nominal
5\% level. The balanced synthetic design does not establish a deployment FPR
guarantee.

For each condition, its 40 candidate scores form the positive sample and the
matched never-stored scores form the negative sample. The area under the
receiver operating characteristic curve (AUC) measures how well the two
score distributions separate. It compares all positive--negative score pairs,
with half credit for ties. Its intervals resample whole evaluation contexts,
keeping each condition and its never-stored reference together. The true
positive rate (TPR) is the fraction of condition scores above the threshold;
the false positive rate (FPR) is the fraction of never-stored scores above it.
Detection intervals additionally resample the calibration contexts
independently, select a new threshold, and apply it to the resampled
evaluation pairs. The two-stage interval for masking's TPR is 5--50\%; for
its paired TPR-minus-FPR difference it is 2.5--37.5 percentage points around
the observed 22.5-point advantage. Conditional Wilson intervals in
Table~\ref{tab:kda-prospective-attack} show rate uncertainty at the deployed
threshold. They also avoid reading a degenerate all-success bootstrap
interval as a population guarantee.

Masking with transport leaves small score separation: AUC is 0.5188, although
its TPR equals the observed FPR at the selected threshold. After completing
the frozen analysis, we checked how matched positive/negative pairs contribute
to empirical AUC. Excluding comparisons from the same source context,
including duplicated sources during bootstrap, gives 0.5115
\([0.5033,0.5238]\) for masking with transport and 0.7333
\([0.6793,0.8007]\) for masking alone. This exploratory calculation is
reported in \path{followup/analysis/auc_pairing_sensitivity_v2.json}; it does
not replace the primary AUC or change the attack. Replay and the independent
rebuild have identical scores to never stored, so their measured AUC is 0.5
and their TPR equals FPR. These are null controls for this fixed attack.

\subsection{Complete active-state comparison}

Replay and the independent never-stored rebuild match the reference in all
80 contexts, including all 40 evaluation contexts. Each comparison covers 80
KDA recurrent/convolution arrays, 14 logical MLA key/value arrays, seven cache
offsets and a full next-token audit logit vector. We compare shapes and the
float32 maximum absolute differences, and check offsets directly. Every
reported difference is zero; unused attention-buffer capacity is excluded.
The all-context success interval is 95.4--100\% by Wilson, while the
evaluation-only interval is 91.2--100\%. The aggregate recurrent residual of
native transport, divided by the original present-minus-never-stored
difference norm, has evaluation mean 0.7011 \([0.6861,0.7146]\).
The complete analysis is \path{followup/analysis/kimi8_analysis_v2.json}.


\begin{table*}[t]
\centering
\footnotesize
\setlength{\tabcolsep}{3pt}
\caption{Target and retained-information scores in the Kimi cohort.}
\label{tab:kda-prospective-cohort}
\begin{tabular}{lccc}
\hline
Condition & \shortstack{Target log-probability\\$\Delta$} & \shortstack{Retained log-probability\\$\Delta\times10^3$} & \shortstack{Continuation NLL\\$\Delta$} \\
\hline
Record present & $3.2399\;[3.0625,\,3.4113]$ & $-0.249\;[-0.453,\,-0.043]$ & $0.0055\;[-0.0083,\,0.0192]$ \\
Attention masking & $0.4858\;[0.4257,\,0.5463]$ & $-0.269\;[-0.461,\,-0.085]$ & $-0.0025\;[-0.0161,\,0.0114]$ \\
Static receipt & $3.1783\;[3.0102,\,3.3414]$ & $0.101\;[-0.125,\,0.302]$ & $-0.0133\;[-0.0315,\,0.0038]$ \\
Diagonal ledger & $3.1884\;[3.0179,\,3.3528]$ & $-0.067\;[-0.223,\,0.078]$ & $-0.0041\;[-0.0193,\,0.0098]$ \\
Receipt transport & $3.1878\;[3.0180,\,3.3514]$ & $-0.069\;[-0.249,\,0.096]$ & $-0.0057\;[-0.0194,\,0.0078]$ \\
Masking + transport & $0.0178\;[0.0108,\,0.0245]$ & $-0.101\;[-0.234,\,0.014]$ & $0.0049\;[-0.0050,\,0.0151]$ \\
Checkpoint replay & $0.0000\;[0.0000,\,0.0000]$ & $0.000\;[0.000,\,0.000]$ & $0.0000\;[0.0000,\,0.0000]$ \\
Never stored & $0.0000\;[0.0000,\,0.0000]$ & $0.000\;[0.000,\,0.000]$ & $0.0000\;[0.0000,\,0.0000]$ \\
\shortstack[l]{Independent\\never stored} & $0.0000\;[0.0000,\,0.0000]$ & $0.000\;[0.000,\,0.000]$ & $0.0000\;[0.0000,\,0.0000]$ \\
Prompt forgetting & $3.1205\;[2.9432,\,3.2908]$ & $0.072\;[-0.230,\,0.384]$ & $0.0124\;[-0.0081,\,0.0339]$ \\
Neutral instruction & $3.2588\;[3.0801,\,3.4306]$ & $-2.878\;[-3.675,\,-2.180]$ & $-0.0356\;[-0.0591,\,-0.0118]$ \\
\hline
\end{tabular}
\par\vspace{3pt}
\begin{minipage}{0.98\textwidth}\footnotesize
Entries are condition-minus-never-stored differences on 40 evaluation contexts, with paired 95\% intervals from 2,000 whole-context resamples. Scores average three target prompts, three retained-record queries, or token-weighted continuation NLL within each context, then give contexts equal weight. Positive target differences indicate residual lift; positive NLL differences indicate worse prediction. Unscaled differences are in nats/token; retained differences are multiplied by $10^3$. The 40 calibration contexts are separate.
\end{minipage}
\end{table*}

\begin{table*}[t]
\centering
\footnotesize
\setlength{\tabcolsep}{3pt}
\caption{Three-call detection results in the Kimi cohort.}
\label{tab:kda-prospective-attack}
\begin{tabular}{lccc}
\hline
Condition & AUC [95\% interval] & \shortstack{TPR, \%\\{[two-stage 95\% interval]}} & \shortstack{TPR Wilson\\95\% interval, \%} \\
\hline
Record present & $1.0000\;[1.0000,\,1.0000]$ & $100.0\;[100.0,\,100.0]$ & $[91.2,\,100.0]$ \\
Attention masking & $0.7400\;[0.6950,\,0.8094]$ & $30.0\;[5.0,\,50.0]$ & $[18.1,\,45.4]$ \\
Static receipt & $1.0000\;[1.0000,\,1.0000]$ & $100.0\;[100.0,\,100.0]$ & $[91.2,\,100.0]$ \\
Diagonal ledger & $1.0000\;[1.0000,\,1.0000]$ & $100.0\;[100.0,\,100.0]$ & $[91.2,\,100.0]$ \\
Receipt transport & $1.0000\;[1.0000,\,1.0000]$ & $100.0\;[100.0,\,100.0]$ & $[91.2,\,100.0]$ \\
Masking + transport & $0.5188\;[0.5113,\,0.5431]$ & $7.5\;[0.0,\,22.5]$ & $[2.6,\,19.9]$ \\
Checkpoint replay & $0.5000\;[0.5000,\,0.5000]$ & $7.5\;[0.0,\,22.5]$ & $[2.6,\,19.9]$ \\
Never stored & $0.5000\;[0.5000,\,0.5000]$ & $7.5\;[0.0,\,22.5]$ & $[2.6,\,19.9]$ \\
\shortstack[l]{Independent\\never stored} & $0.5000\;[0.5000,\,0.5000]$ & $7.5\;[0.0,\,22.5]$ & $[2.6,\,19.9]$ \\
Prompt forgetting & $1.0000\;[1.0000,\,1.0000]$ & $100.0\;[100.0,\,100.0]$ & $[91.2,\,100.0]$ \\
Neutral instruction & $1.0000\;[1.0000,\,1.0000]$ & $100.0\;[100.0,\,100.0]$ & $[91.2,\,100.0]$ \\
\hline
\end{tabular}
\par\vspace{3pt}
\begin{minipage}{0.98\textwidth}\footnotesize
The threshold is $-2.560180$ nats/token and is calibrated using only the 40 never-stored calibration scores. The attacker sees candidate-continuation log probabilities, not model state or evaluation never-stored references. Replay and independent rebuilding are measured null controls: their AUC is $0.5$ and their TPR equals FPR. Degenerate bootstrap intervals at zero or one describe the observed cohort; the Wilson intervals expose finite-sample rate uncertainty. This is one fixed attack on a synthetic cohort, not a general indistinguishability guarantee.
\end{minipage}
\end{table*}

\subsection{Independent Qwen confirmation}\label{app:qwen-cohort}

The Qwen cohort comprises 32 synthetic records with unique target phrases,
frozen before model evaluation and independent of the three mechanistic
Qwen scenarios. Prefixes span eight domains
and contain two to five retained entries. Prefix lengths are 91--191 tokens
and target records are 37--41 tokens; each target receives both a short suffix
of 90--94 tokens and a long suffix of 527--539 tokens. Raw concatenated text
and segment tokenizations agree at every boundary. No qualification filter,
target replacement or outcome-dependent stopping is used.

The model is Qwen3.5-4B-Base at revision
\texttt{57370f0e\allowbreak{}a82c3cca\allowbreak{}33558a95\allowbreak{}212e032c\allowbreak{}344e5fd5}, run in bf16 with
Transformers 5.12.1. Each suffix evaluation advances the record-present,
checkpoint-replayed, fresh never-stored and independently repeated
never-stored trajectories. Replay and both fresh references use identical
surviving tokens and segment schedules. All 64 replay and all 64 repeat
comparisons have exact numerical equality on 65 arrays or outputs and 82
bookkeeping flags, covering every active recurrent and attention layer,
final logits and logical cache position. Arrays are checked with
\texttt{torch.equal} and zero maximum absolute difference; this is not a
byte-level comparison of representations.

We inspect transport and forcing at recurrent layers 0, 16 and 30. Inputs are
captured before the live kernel, permitting separate checks of the handwritten
recurrence, frozen-input transport identity and native forcing decomposition.
The largest state-scaled frozen-transport error is \(1.047\times10^{-6}\),
live-recurrence conformance error is \(1.190\times10^{-6}\), and relative
forcing-decomposition error is \(5.935\times10^{-6}\); all are below the
frozen acceptance threshold of 0.005. The native-mismatch check requires a
nonzero native record effect and at least one selected layer above five times
its empirical floor. That floor is the maximum of the absolute frozen
identity error, handwritten-recurrence/live-kernel mismatch and
\(10^{-12}\). All 192 selected layer--context--suffix checks exceed five
times the floor. All 32 records pass every frozen check for both suffixes;
the record-level Wilson 95\% interval is 89.3--100\%.

Table~\ref{tab:kda-qwen-cohort} reports means and 95\% intervals. We average
layers 0, 16 and 30 within each record before resampling the 32 records;
suffixes and layers remain paired. Effect scaling divides the residual norm
by the native record-present-minus-never-stored difference norm. It asks how
large the remaining error is relative to the record's original effect. State
scaling divides by the never-stored state norm, comparing the error with the
size of the reference state.
The native residual grows relative to the native record effect for the long
suffix, but is smaller relative to the never-stored state norm. The suffix
texts differ, so this paired contrast compares two conditions rather than
isolating a causal effect of length. Reporting both normalizations also
prevents interpreting it as universal worsening. In the long suffixes a
frozen record effect can approach the
floating-point floor; the frozen-transport decision therefore uses
state scaling and the native decision uses its absolute numerical control.
The aggregate evidence is \path{followup/analysis/qwen_summary.json}.

\begin{table*}[t]
\centering
\footnotesize
\setlength{\tabcolsep}{3pt}
\caption{Qwen transport residuals with short and long suffixes.}
\label{tab:kda-qwen-cohort}
\begin{tabular}{@{}lccc@{}}
\toprule
quantity & short suffix & long suffix & paired long minus short \\
\midrule
\shortstack[l]{native transport\\effect scaled} &
\(0.5222\ [0.5151,0.5296]\) & \(0.5803\ [0.5760,0.5847]\) & \(0.0581\ [0.0506,0.0655]\) \\
\shortstack[l]{native transport\\state scaled} &
\(0.0946\ [0.0914,0.0978]\) & \(0.0878\ [0.0856,0.0903]\) & \(-0.0067\ [-0.0104,-0.0032]\) \\
\shortstack[l]{frozen static\\state scaled} &
\(0.2884\ [0.2850,0.2917]\) & \(0.3059\ [0.3022,0.3093]\) & \(0.0174\ [0.0161,0.0187]\) \\
\shortstack[l]{frozen ledger\\state scaled} &
\(0.1032\ [0.1019,0.1045]\) & \(0.0892\ [0.0881,0.0903]\) & \(-0.0140\ [-0.0148,-0.0132]\) \\
\bottomrule
\end{tabular}
\end{table*}

\subsection{Kimi precision comparison using the same source weights}\label{app:precision-control}

This control asks whether native transport mismatch persists when we change
precision while keeping the source weights fixed. We loaded the original
Kimi bf16 weights and derived a paired 8-bit model from them. All 20 weight
shards were verified against their SHA-256 values at revision
\texttt{e1df551a\allowbreak{}447157d4\allowbreak{}658b573f\allowbreak{}9a695d57\allowbreak{}658590e9} of
\path{moonshotai/Kimi-Linear-48B-A3B-Instruct}. The source contains original
bf16 weights and fp32 parameters, rather than dequantized 8-bit weights.
The derived variant uses MLX affine weight quantization with eight bits,
group size 64 and the Kimi model's quantization predicate through a
supported-module/divisibility check. Quantization occurs in memory before
attention-mask installation. The comparison concerns this specified recipe
and its runtime kernels.

Eight indices from the main evaluation manifest were selected before reading
their outcomes: 000, 001, 002, 003, 006, 007, 008 and 009. Four have 128-token
suffixes and four have 1,024-token suffixes. Context segments, probes, targets,
instructions and stop-token IDs match exactly across the paired tokenizers,
and observed token plans agree. Each precision is compared with its own
never-stored reference. All eight contexts qualify in each precision, and
replay and repeated rebuilding match all 94 active arrays, seven offsets and
the audit logits in every context. With eight contexts, the Wilson 95\%
interval for an all-success rate is 67.6--100\%.

For an independent arithmetic diagnostic, we capture each native branch's
\(k,v,g,\beta\) and reconstruct its recurrent endpoints in serial fp32.
We also propagate two initial states through identical present-branch inputs
and writes and compare their difference with a separately transported receipt.
We use these checks to estimate the numerical error introduced by the
arithmetic. This empirical error floor is the larger of the summed
native-versus-serial endpoint reconstruction norms and the frozen-input
identity-error norm. It measures
observed arithmetic discrepancies in these executions, not a rigorous error
bound. The global transport residual is 21,928--84,695 times its own floor in
derived 8-bit and 20,970--80,373 times its own floor in bf16. Maximum global
endpoint reconstruction relative errors are approximately
\(1.1\times10^{-6}\) in both. Three of the 160 layer--context residuals
are at or below their own floor in each precision, so the aggregate result
does not assert a mismatch above the floor in every layer.

The native residual divided by the original state difference ranges from
0.665 to 0.780 in derived 8-bit and from 0.645 to 0.817 in original bf16.
Table~\ref{tab:kda-precision-cohort} gives the paired target log-probability
lifts in nats per token. Each precision uses its own never-stored reference.
The 95\% intervals use 2,000 paired whole-context resamples; the last column
compares these own-reference lifts. The
same-master comparison supports persistence of the native transport mismatch
above measured arithmetic discrepancies. It does not require readout values
to be equivalent across precisions: masking with transport has a
bf16-minus-8-bit change of \(-0.031\) nats per token with interval
\([-0.054,-0.009]\). Retained and continuation measurements for all eleven
conditions are in \path{followup/analysis/precision_comparison_v1.json}.
The separately pinned community-model comparison remains in
\path{followup/analysis/precision_cached8_comparison_v1.json} as a secondary
artifact comparison; its exact master-checkpoint relationship is unverified.

\begin{table*}[t]
\centering
\small
\caption{Kimi target lift at two precisions using the same source weights.}
\label{tab:kda-precision-cohort}
\begin{tabular}{@{}lccc@{}}
\toprule
condition & derived 8-bit lift & original bf16 lift & bf16 minus 8-bit lift \\
\midrule
present & \(3.287\ [2.793,3.810]\) & \(3.255\ [2.773,3.756]\) & \(-0.032\ [-0.076,0.003]\) \\
attention masking & \(0.426\ [0.268,0.580]\) & \(0.403\ [0.227,0.583]\) & \(-0.023\ [-0.068,0.017]\) \\
receipt transport & \(3.239\ [2.766,3.753]\) & \(3.211\ [2.739,3.700]\) & \(-0.028\ [-0.073,0.010]\) \\
masking + transport & \(0.016\ [-0.001,0.034]\) & \(-0.015\ [-0.040,0.011]\) & \(-0.031\ [-0.054,-0.009]\) \\
checkpoint replay & \(0\ [0,0]\) & \(0\ [0,0]\) & \(0\ [0,0]\) \\
prompt forgetting & \(3.159\ [2.668,3.680]\) & \(3.140\ [2.653,3.632]\) & \(-0.019\ [-0.073,0.028]\) \\
neutral instruction & \(3.313\ [2.815,3.836]\) & \(3.286\ [2.798,3.789]\) & \(-0.027\ [-0.072,0.015]\) \\
\bottomrule
\end{tabular}
\end{table*}

\subsection{Amendment and sequential maintenance}\label{app:maintenance}

A correction needs to install the new answer as well as remove the old one.
We tested twelve frozen synthetic amendment contexts with varied target
positions among four records. Six contexts have no extra neutral suffix,
three add 128 tokens and three add 512. All old targets have record-present
lift above one nat; none is excluded. The maintenance manifest's payload SHA-256, computed from sorted-key JSON
before adding its \texttt{sha256} field, is
\texttt{50edd37c\allowbreak{}5f7d060d\allowbreak{}aeb976e5\allowbreak{}39abffbc\allowbreak{}421cbf7f\allowbreak{}f2ef10a9\allowbreak{}f4ecd7be\allowbreak{}d01b45e7}.
We restore the checkpoint before the target, ingest its replacement and replay
the surviving records. A fresh corrected-from-start context supplies the
reference. Greedy decoding has a 16-token cap. Exact-answer and containment
rules follow the main cohort, with punctuation preserved.

All twelve contexts return the old answer before amendment and the new answer
afterward; none returns the new answer before amendment or the old answer
afterward. Retained answers match exactly in 35 of 36 queries, with all three
exact in 11 of 12 contexts; all 36 contain their targets. These outcomes are
grouped by context for uncertainty in Table~\ref{tab:kda-maintenance-cohort}.
Paired whole-context bootstrap intervals give an old-answer score change of
\(-6.8277\ [-7.1725,-6.4468]\), a new-answer change of
\(7.1025\ [6.4566,7.8272]\), and a mean retained-score change of
\(-0.0019\ [-0.0087,0.0023]\) nats per token relative to the original
context. All three score changes relative to the corrected-from-start
reference are zero.

We next form six disjoint pairs from these twelve targets and test both
deletion orders. Each path deletes one target, adds a new record, then
deletes the other target. For a pair of targets \(A\) and \(B\), we therefore
test both \(A\)-then-\(B\) and \(B\)-then-\(A\), adding the new record between
the two deletions. Thus there are twelve paths and 36 transitions,
with a fresh surviving-record reference at every transition. Both final
orders have identical surviving-record keys and saved probe scores in all
six pairs. Both deleted answers are absent in all pairs. Every retained and
newly admitted answer is exact in both final orders for five pairs; all such
answers contain their targets in all six. The independent sequential unit is
the pair, not the path or transition, and these checks do not enlarge the
twelve-target cohort.

All 48 intervention-to-reference comparisons (twelve amendments and 36
transitions) and all 48 independent reference repeats have zero measured
state and logit differences. Each comparison checks dtypes, shapes, all 94
active arrays, seven MLA offsets and a full next-token logit vector.
The evidence is \path{followup/analysis/maintenance_summary_v1.json}.
These runs overlapped other experiments and a weight download; their times
are not used as controlled latency measurements. Table~\ref{tab:kda-maintenance-cohort}
reports Wilson 95\% intervals. Amendment outcomes use 12 contexts; sequence
outcomes use six disjoint target pairs, with success required in both orders.

\begin{table*}[t]
\centering
\small
\caption{Record correction and sequential-deletion outcomes.}
\label{tab:kda-maintenance-cohort}
\begin{tabular}{@{}p{0.56\linewidth}cc@{}}
\toprule
outcome & count & interval, \% \\
\midrule
new answer exact after amendment & 12/12 & 75.8--100 \\
old answer exact after amendment & 0/12 & 0--24.2 \\
all retained answers exact after amendment & 11/12 & 64.6--98.5 \\
all retained targets contained after amendment & 12/12 & 75.8--100 \\
all sequence transitions match fresh state/logits & 6/6 & 61.0--100 \\
both deleted answers absent in both orders & 6/6 & 61.0--100 \\
all retained/admitted answers exact in both orders & 5/6 & 43.6--97.0 \\
all retained/admitted targets contained in both orders & 6/6 & 61.0--100 \\
\bottomrule
\end{tabular}
\end{table*}

%% file: kda_historical_tables.tex
\subsection{Corpus, suffix and cross-family audit tables}\label{app:initial-tables}

The mechanism audits use the corpora and qualification rules in the following
tables. They are separate from the prospectively selected synthetic cohort
in Appendix~\ref{app:prospective}.

Table~\ref{tab:audit-denominators} separates all attempted cases from those
that met the recall criterion before deletion. Qualification requires the
record-present field lift to reach the fixed \(0.05\)-nat threshold and,
for short categorical targets, the first token to rank among the top ten.
Cases below this threshold remain in the attempted count.

\begin{table*}[t]
\centering
\caption{Attempted and qualified cases in the mechanism audits.}
\label{tab:audit-denominators}
\footnotesize
\setlength{\tabcolsep}{3pt}
\begin{tabular}{@{}p{0.23\linewidth}p{0.17\linewidth}p{0.10\linewidth}p{0.40\linewidth}@{}}
\toprule
surface & attempted & qualified & audited unit \\
\midrule
Kimi short tables & 8 records & 8 &
final logits and 80 KDA arrays per record \\
Kimi 128-record context & 9 records & 9 &
same exactness surface plus nine position timings \\
Kimi long notes & 5 records & 4\textsuperscript{a} &
same exactness surface \\
Kimi fixed-position masking vs.\ replay & 12 records (3 per corpus, 4 corpora) &
11\textsuperscript{b} &
lift after masking and after replay; replay residual on logits and KDA state \\
Kimi suffix sweep, synthetic & \SweepSynAttempted{} configurations &
\SweepSynAdmitted{} &
20 layers at \SweepSynNumCuts{} suffix lengths; probes and \SampSynK{} sampled draws per condition \\
Kimi suffix sweep, TOFU & \SweepTofuAttempted{} configurations &
\SweepTofuAdmitted{} & same surface \\
Qwen staged check & 3 scenarios, 6 paths & 6 &
65 arrays/outputs and 82 flags across the six paths \\
Mamba-2 / Falcon-H1 / RWKV-7 &
\FamMambaAttempted{} / \FamFalconAttempted{} / \FamRwkvAttempted{} configurations &
\FamMambaAdmitted{} / \FamFalconAdmitted{} / \FamRwkvAdmitted{} &
all recurrent layers at five suffix lengths; probes at every length \\
\bottomrule
\end{tabular}

{\footnotesize\textsuperscript{a}The excluded record is the oldest MIMIC-Note
record; its record-present lift of \(0.00451\) nats falls below the frozen
threshold, so deletion is not scored for it.
\textsuperscript{b}The early MIMIC-Note record of the fixed-position runs has
a record-present lift of \(0.004\) nats, also below the threshold; its
masking result is not scored and its replay residual is reported
(Table~\ref{tab:kimi-oracle}).\par}
\end{table*}

Table~\ref{tab:kda-behavior} compares seven conditions, separately for each
corpus. Lift, first-token rank and first-token Kullback--Leibler (KL)
divergence are medians. The greedy column reports the percentage of
configurations whose answer contains the target string, ignoring
capitalization. Lift and KL are in nats. KL compares the first-token
distributions as
\(D_{\rm KL}(p_{\rm never\text{-}stored}\Vert p_{\rm condition})\).

\begin{table*}[t]
\centering
\caption{Target recovery after 4,096 suffix tokens.}
\label{tab:kda-behavior}
\small
\setlength{\tabcolsep}{3pt}
\begin{tabular}{@{}lrrrrrrrr@{}}
\toprule
& \multicolumn{4}{c}{Synthetic: 24 configurations} & \multicolumn{4}{c}{TOFU: 16 configurations} \\
\cmidrule(lr){2-5}\cmidrule(l){6-9}
condition & lift & rank & greedy \% & KL & lift & rank & greedy \% & KL \\
\midrule
record present & 2.16 & 1 & 100 & 4.55 & 1.34 & 1 & 25 & 0.549 \\
static receipt & 2.15 & 1 & 100 & 4.64 & 1.31 & 1 & 25 & 0.457 \\
diagonal ledger & 2.16 & 1 & 100 & 4.52 & 1.33 & 1 & 25 & 0.563 \\
full-matrix receipt & 2.15 & 1 & 100 & 4.62 & 1.34 & 1 & 25 & 0.536 \\
attention masked & 0.0622 & 89 & 0 & 0.0354 & 0.0472 & 1 & 0 & 0.00419 \\
masked + full-matrix receipt & -0.0115 & 138 & 0 & 0.0435 & 0.0154 & 1 & 0 & 0.00425 \\
never-stored state & 0 & 167 & 0 & 0 & 0 & 1 & 0 & 0 \\
\bottomrule
\end{tabular}
\end{table*}

Table~\ref{tab:kda-sampling} reports target hits from 100 draws per
configuration at temperature 1, separately at ingestion and after 4,096
suffix tokens. It compares five conditions. The combined intervention uses
a static receipt, not the full-matrix receipt in Table~\ref{tab:kda-behavior}.

\begin{table*}[t]
\centering
\caption{Target hits in sampled answers.}
\label{tab:kda-sampling}
\small
\setlength{\tabcolsep}{6pt}
\begin{tabular}{@{}lrrrr@{}}
\toprule
& \multicolumn{2}{c}{Synthetic: 2,400 draws} & \multicolumn{2}{c}{TOFU: 1,600 draws} \\
\cmidrule(lr){2-3}\cmidrule(l){4-5}
condition & 0 tokens & 4,096 & 0 tokens & 4,096 \\
\midrule
record present & 2349 & 2287 & 358 & 325 \\
static receipt & 2314 & 2278 & 245 & 326 \\
attention masked & 1 & 1 & 2 & 11 \\
masked + static receipt & 1 & 0 & 1 & 10 \\
never-stored state & 1 & 0 & 0 & 12 \\
\bottomrule
\end{tabular}
\end{table*}

Table~\ref{tab:kimi-oracle} compares masking and replay on the same records.
These cases are separate from the 21 primary corpus replay checks. Lift is
measured in nats against the never-stored state. The replay residual is the
maximum absolute difference over final logits and 80 KDA arrays.

\begin{table*}[t]
\centering
\caption{Attention masking and replay at three record positions.}
\label{tab:kimi-oracle}
\footnotesize
\setlength{\tabcolsep}{5pt}
\begin{tabular}{@{}ll S[table-format=1.2] S[table-format=-1.2] S[table-format=1.0] S[table-format=4.0]@{}}
\toprule
corpus & position & {present lift} & {lift after mask} &
{replay residual} & {suffix tokens} \\
\midrule
TOFU & early & 0.78 & -0.08 & 0 & 762 \\
TOFU & middle & 1.51 & 0.11 & 0 & 399 \\
TOFU & late & 1.50 & 0.19 & 0 & 63 \\
\addlinespace
MIMIC-CDS & early & 0.77 & 0.02 & 0 & 366 \\
MIMIC-CDS & middle & 1.97 & 0.24 & 0 & 176 \\
MIMIC-CDS & late & 0.61 & 0.19 & 0 & 26 \\
\addlinespace
MIMIC-Note & early\textsuperscript{a} & 0.00 & 0.00 & 0 & 6635 \\
MIMIC-Note & middle & 2.33 & -0.02 & 0 & 3135 \\
MIMIC-Note & late & 1.75 & -0.04 & 0 & 1017 \\
\bottomrule
\end{tabular}
\par\vspace{3pt}
\begin{minipage}{0.98\textwidth}\footnotesize
\textsuperscript{a}The early MIMIC-Note record falls below the 0.05-nat recall
threshold. Its masking result is not scored; the displayed 0.00 is a placeholder,
not a measured zero lift. Its replay residual is reported.
\end{minipage}
\end{table*}

Table~\ref{tab:kda-replay-checks} summarizes checks with different cohorts
and metrics; its two rows are not a paired comparison. Corpus replay uses
the cohorts in Table~\ref{tab:audit-denominators}. Receipt checks use six
configurations, three layers and two suffixes, giving 36 layer--suffix checks.

\begin{table*}[t]
\centering
\caption{Replay and receipt checks: cohorts and outcomes.}
\label{tab:kda-replay-checks}
\small
\setlength{\tabcolsep}{5pt}
\begin{tabular}{@{}lll@{}}
\toprule
correction & population & result \\
\midrule
checkpoint replay & 21 qualified records & logits and 80 arrays: residual 0 \\
full-matrix receipt & 36 layer--suffix checks & forcing norm ratio 0.064--1.010 \\
\bottomrule
\end{tabular}
\end{table*}